\documentclass[letterpaper]{article}

\usepackage{microtype}
\usepackage{graphicx}

\usepackage{epstopdf}
\DeclareGraphicsExtensions{.pdf,.eps}

\usepackage{subfigure}
\usepackage{booktabs} %

\usepackage{hyperref}

\usepackage[accepted]{icml2020}

\icmltitlerunning{On Convergence-Diagnostic based Step Sizes for Stochastic Gradient Descent}

\usepackage{dsfont}
\usepackage{amsfonts}
\usepackage{natbib}
\usepackage{amsmath}
\usepackage{amsthm}
\usepackage{url}
\usepackage{stmaryrd}
\usepackage{amssymb} %
\usepackage{bbm}
\usepackage{algorithm}
\usepackage{algorithmic}
\usepackage{dirtytalk}



\newcommand{\norm}[1]{\left\lVert#1\right\rVert}
\newcommand{\R}{\mathbb{R}}
\newcommand{\Prob}[1]{\mathbb{P}\left( #1 \right)}

\newcommand{\ProbRest}[1]{\mathbb{ P}_{\theta_0 \sim \pi_{\gamma_{old}}}\left( #1 \right)}

\newcommand{\C}{\mathcal{C}}
\newcommand{\ps}[2]{\langle #1 , \ #2\rangle}
\newcommand{\tr}[1]{\operatorname{Tr } \ #1}
\newcommand{\trsq}[1]{\operatorname{Tr }^2 \ #1}
\newcommand{\E}[1]{\mathbb{E} \left [ #1 \right]}
\newcommand{\Epi}[1]{\mathbb{E}_{\pi_{\gamma}} \left [ #1 \right ]}
\newcommand{\Teta}{\Tilde{\eta}}
\newcommand{\Erestart}[1]{\mathbb{E}_{\theta_0 \sim \pi_{\gamma_{old}}}\left [ #1 \right ]}
\newcommand{\Var}[1]{\operatorname{Var}( #1 )}
\newcommand{\pigamma}{\pi_{\gamma}}
\newcommand{\F}{\mathcal{F}}
\usepackage{enumitem}

\usepackage{cleveref}
\crefname{lemma}{Lemma}{Lemmas}
\crefname{appendix}{Appendix}{Appendices}

\crefname{fact}{Fact}{Facts}
\crefname{theorem}{Theorem}{Theorems}
\crefname{corollary}{Corollary}{Corollaries}
\crefname{claim}{Claim}{Claims}
\crefname{example}{Example}{Examples}
\crefname{problem}{Problem}{Problems}
\crefname{definition}{Definition}{Definitions}
\crefname{assumption}{Assumption}{Assumptions}
\crefname{subsection}{Subsection}{Subsections}
\crefname{section}{Section}{Sections}
\crefname{algorithm}{Algorithm}{Algorithms}
\crefname{algocf}{alg.}{algs.}
\Crefname{algocf}{Algorithm}{Algorithms}
\crefname{proposition}{Proposition}{Propositions}
\crefname{figure}{Fig.}{Figs.}

\newtheorem{theorem}{Theorem}
\newtheorem{lemma}[theorem]{Lemma}
\newtheorem{proposition}[theorem]{Proposition}

\newtheorem{corollary}[theorem]{Corollary}

\newtheorem{assumption}[theorem]{Assumption}

\newenvironment{sproof}{%
  \proof}{\endproof}

\usepackage{xr}
\makeatletter
\newcommand*{\addFileDependency}[1]{%
  \typeout{(#1)}
  \@addtofilelist{#1}
  \IfFileExists{#1}{}{\typeout{No file #1.}}
}
\makeatother
 
\newcommand*{\myexternaldocument}[1]{%
    \externaldocument{#1}%
    \addFileDependency{#1.tex}%
    \addFileDependency{#1.aux}%
}

\myexternaldocument{Main-with-appendix}

\let\oldref\ref
\newcommand{\sref}[1]{(\oldref{#1})}

\begin{document}

\twocolumn[
\icmltitle{On Convergence-Diagnostic based Step Sizes for Stochastic Gradient Descent
}

\icmlsetsymbol{equal}{*}

\begin{icmlauthorlist}
\icmlauthor{Scott Pesme}{Scott}
\icmlauthor{Aymeric Dieuleveut}{Aymeric}
\icmlauthor{Nicolas Flammarion}{Scott}
\end{icmlauthorlist}

\icmlaffiliation{Scott}{Theory of Machine Learning lab, EPFL}
\icmlaffiliation{Aymeric}{École Polytechnique}

\icmlcorrespondingauthor{Scott Pesme}{scott.pesme@epfl.ch}

\icmlkeywords{Machine Learning, ICML}

\vskip 0.3in
]

\printAffiliationsAndNotice{}  %

\begin{abstract}

Constant step-size Stochastic Gradient Descent exhibits  two phases: a transient phase during which iterates make fast progress towards the optimum, followed by a stationary phase during which iterates oscillate around the optimal point. In this paper, we show that efficiently detecting this transition and appropriately decreasing the step size can lead to fast convergence rates. We analyse the classical statistical test proposed by \citet{pflug1983determination}, based on the inner product between consecutive stochastic gradients. Even in the simple case where the objective function is quadratic we show that this test cannot lead to an adequate convergence diagnostic. We then propose a novel and simple statistical procedure that accurately detects stationarity and we provide experimental results showing state-of-the-art performance on synthetic and real-world datasets.

\end{abstract}

\section{Introduction}

The field of machine learning has had tremendous success in recent years, in problems such as object classification~\cite{he2016deep} and speech recognition~\cite{graves2013speech}. These achievements have been enabled by the development of complex optimization-based architectures such as deep-learning, which are efficiently trainable by Stochastic Gradient Descent algorithms \cite{bottou98}. 

Challenges have arisen on both the theoretical front -- to understand why those algorithms achieve such performance, and on the practical front -- as choosing the architecture of the network and the parameters of the algorithm has become an art itself. Especially, there is no practical heuristic to set the step-size sequence. As a consequence, new optimization strategies have appeared to alleviate the tuning burden, as Adam \cite{kingma2014adam},
together with new learning rate scheduling, such as cyclical learning rates \cite{smith2017cyclical} and warm restarts \cite{Loshchilov2016}. However those strategies typically do not come with theoretical guarantees and may be outperformed by  SGD~\cite{wilson2017marginal}.

Even in the classical case of convex optimization, in which convergence rates have been widely studied over the last 30 years~\cite{polyak1992acceleration,zhang2004solving,nemirovski2009robust,moulines2011non,RakShaSri12} and where theory suggests to use the \emph{averaged iterate} and provides optimal choices of learning rates, practitioners still face major challenges: indeed (a) averaging leads to a slower decay during early iterations, (b) learning rates  may not adapt to the difficulty of the problem (the optimal decay depends on the class of problems), or may not be robust to constant misspecification.
Consequently, the state of the art approach in practice remains to use the \emph{final iterate} with decreasing step size $a/(b+ t^\alpha)$ with  constants $a,b,\alpha$ obtained by  a 
tiresome hand-tuning.
Overall, there is a desperate need for adaptive algorithms.

 In this paper, we study \emph{adaptive step-size scheduling} based on \emph{convergence diagnostic}. 
 The behaviour of SGD  with constant step size is dictated by (a) a \emph{bias term}, that accounts for the impact of the initial distance $\norm{\theta_0 -\theta_*}$ to the minimizer $\theta_*$ of the function, and (b) a \emph{variance term} arising from the noise in the gradients.
Larger steps allow to forget the initial condition faster, but increase the impact of the noise.
Our approach is then to use the largest possible learning rate as long as the iterates make progress and to \emph{automatically} detect when they stop making any progress. When we have reached such a \emph{saturation}, we reduce the learning rate.  This can be viewed as  ``restarting'' the algorithm, even though only the learning rate changes. We refer to this approach as \emph{Convergence-Diagnostic algorithm}. 
Its benefits are thus twofold: (i) with a large initial learning rate the bias term initially decays at an \emph{exponential rate}~\cite{kushner1981asymptotic,pflug1986stochastic}, (ii) decreasing the learning rate when  the effect of the noise becomes dominant defines an  efficient and practical adaptive strategy.

Reducing the learning rate when the objective function stops decaying is widely used in deep learning~\cite{krizhevsky2012imagenet} but the epochs where the step size is reduced are mostly 
hand-picked. Our goal is to select them  automatically by detecting saturation. Convergence diagnostics date back to  \citet{pflug1983determination}, who proposed to use the inner product between consecutive gradients to detect convergence. Such a strategy has regained interest in recent years: ~\citet{chee2018convergence} provided a similar analysis for quadratic functions, and \citet{yaida18} considers SGD with momentum and proposes an analogous restart criterion using the expectation of an observable quantity under the limit distribution, achieving  the same performance as hand-tuned methods on two simple deep learning models. However, none of these papers provide a convergence rate and we show that Pflug's  approach provably fails in simple settings.
\citet{lang2019using} introduced Statistical Adaptive Stochastic Approximation
which aims to improve upon Pflug's approach by formalizing the testing procedure. However, their strategy leads to a very small number of reductions of the learning rate.

An earlier attempt to adapt the  learning   rate depending on the directions in which iterates are moving was made by \citet{kesten1958accelerated}. Kesten's rule decreases the step size when the iterates stop moving consistently in the same direction.  Originally introduced in one dimension, it was generalized to the multi-dimensional case and analyzed by \citet{delyon1993accelerated}.

 Finally, some orthogonal  approaches have also been used to automatically change the learning rate: it is for example possible to consider the step size as a parameter of the risk of the algorithm, and to update the step size using another meta-optimization algorithm~\cite{sutton1981adaption,jacob1988increased,benveniste1990adaptive,sutton1992adapting, schraudolph1999local, kushner1995analysis,almeida1999parameter}. 

Another line of work consists in changing the learning rate for each coordinate depending on how much iterates are moving~\cite{duchi2011adaptive,matthew12}. 
Finally, \citet{schaul2013no} propose to use coordinate-wise adaptive learning rates, that maximize the decrease of the expected loss on separable quadratic functions.

We make the following contributions:
\begin{itemize}[leftmargin=*,itemsep=0pt,topsep=0pt]%
    \item We provide convergence results for the Convergence-Diagnostic algorithm  when used with  the oracle diagnostic for smooth and strongly-convex functions. 
    \item We show that the intuition for Pflug's statistic is valid for all smooth and strongly-convex functions
  by computing the expectation 
    of the inner product between two consecutive gradients both for an arbitrary starting point, and under the stationary distribution. 
    \item We show that despite the previous observation the empirical criterion is provably inefficient, even for a simple quadratic objective.
    \item We introduce a new \emph{distance-based diagnostic} based on a simple heuristic inspired from the quadratic setting with additive noise.
    \item We illustrate experimentally the failure of Pflug's statistic, and show that the distance-based diagnostic competes with state-of-the-art methods on a variety of loss functions, both on synthetic and real-world datasets.
\end{itemize}
The paper is organized as follows: in \Cref{sec:preliminaries}, we introduce the framework and present the assumptions. \cref{sec:BiasVariance} we describe and analyse the oracle convergence-diagnostic algorithm. In \cref{sec:pflug}, we show that the classical criterion proposed by Pflug cannot efficiently detect stationarity. We then introduce a new distance-based criterion \cref{sec:new_stat} and provide numerical experiments in  \cref{sec:experiments}.

\section{Preliminaries}\label{sec:preliminaries}
Formally, we consider the minimization of a risk function $f$ defined on $\mathbb R^d$ given access to a sequence of unbiased estimators of $f$'s gradients~\cite{robbins1951stochastic}.  
Starting from an arbitrary point $\theta_0$, at each iteration $i + 1$ we get an unbiased random estimate $f'_{i + 1} (\theta_{i})$ of the gradient $f'(\theta_{i})$ and update the current estimator by moving in the opposite direction of the stochastic gradient:
\begin{equation}
\label{eq:sgd}
    \theta_{i+1} =  \theta_i - \gamma_{i+1} f'_{i + 1} (\theta_i),
\end{equation}
where $\gamma_{i+1}> 0$ is the \emph{step size}, also referred to as \emph{learning rate}. We make the following assumptions on the stochastic gradients and the function $f$. 

\begin{assumption}[Unbiased gradient estimates] 
\label{as:unbiased_gradients} 
There exists a filtration $(\F_i)_{i \geq 0}$ such that $\theta_0$ is $\F_0$-measurable, $f'_{i}$ is $\mathcal{F}_i$-measurable for all $i\in \mathbb N$, and for each $\theta \in \mathbb R ^d$: 
$\E{f'_{i + 1} (\theta) \ | \ \mathcal{F}_i} = f'(\theta)$. In addition $(f_i)_{i\geq0}$ are identically distributed random fields.
\end{assumption}

\begin{assumption}[L-smoothness] 
\label{as:smoothness}
For all $i \geq 1$, the function $f_i$ is almost surely $L$-smooth and convex:
$$\forall \theta, \eta \in \mathbb R^d, \norm{f'_i(\theta) - f'_i(\eta)} \leq L \norm{\theta - \eta}.$$
\end{assumption}

\begin{assumption}[Strong convexity] 
\label{as:sc}
There exists a finite constant  $\mu > 0$ such that for all $\theta, \eta \in \mathbb R ^d$:
\[
f(\theta) \geq f(\eta) + \ps{f'(\eta)}{\theta - \eta} + \frac{\mu}{2} \norm{\theta - \eta}^2.
\]
\end{assumption}
For $i > 0$ and $\theta \in \mathcal{W}$, we denote by $\varepsilon_{i}(\theta) = f'_i(\theta) - f'(\theta)$ the noise, for which we consider the following assumption:
\begin{assumption}[Bounded variance] 
\label{as:noise_condition}
There exists a constant $\sigma \geq 0$ such that for any $i > 0$, $\E{ \norm{\varepsilon_i (\theta^*)}^2 } \leq \sigma^2$.
\end{assumption}
Under \cref{as:unbiased_gradients,as:noise_condition} we define the noise covariance as  the function $\mathcal{C} \ : \ \R^d \mapsto \R^{d \times d}$ defined for all $\theta \in \R^d$ by $\C(\theta) = \E{\varepsilon(\theta) \varepsilon(\theta)^T }$.

In the following section we formally describe the restart strategy and give a convergence rate in the omniscient setting where all the parameters are known.

\section{Bias-variance decomposition and stationarity diagnostic}
\label{sec:BiasVariance}

When the step size $\gamma$ is constant, the sequence of iterates $(\theta_n)_{n \geq 0}$ produced by the SGD recursion in \cref{eq:sgd} is a homogeneous Markov chain. Under appropriate conditions~\citep{dieuleveut2017bridging}, this Markov chain has a unique stationary distribution, denoted by $\pigamma$, towards which it converges exponentially fast. This is the \emph{transient phase}. The rate of convergence is proportional to $\gamma$ and therefore a larger step size leads to a faster convergence. 

When the Markov chain has reached its stationary distribution, 
i.e. in the \emph{stationary phase}, the iterates make negligible progress towards the optimum $\theta^*$ but stay in a bounded region of size $O(\sqrt{\gamma})$ around it. More precisely, \citet{dieuleveut2017bridging}  make explicit the expansion $\Epi{\norm{\theta - \theta^*}^2} = b \gamma + O(\gamma^2)$ where the constant $b$ depends on the function $f$ and on the covariance of the noise $\mathcal{C}(\theta^*)$ at the optimum . Hence the smaller the step size and the closer the iterates $(\theta_n)_{n \geq 0}$ get to the optimum $\theta^*$.

Therefore a clear trade-off appears between: (a) using a large step size with a fast transient phase but a poor approximation of $\theta^*$ and (b) using a small step size with iterates getting close to the optimum but taking longer to get there.  This \emph{bias-variance} trade-off is directly transcribed in the following classical proposition~\cite{needell2014stochastic}.
\begin{proposition} 
\label{eq:needell}
 Consider the recursion in \cref{eq:sgd} under \cref{as:unbiased_gradients,as:smoothness,as:sc,as:noise_condition}. Then for any step-size $\gamma \in (0, 1/2L)$ and $n \geq 0$ we have:
\begin{equation} \nonumber
\E{\norm{\theta_n - \theta^*}^2} 
\leq (1 -  \gamma \mu )^n \E{\norm{\theta_0 - \theta^*}^2} + \frac{2\gamma \sigma^2}{\mu}.
\end{equation}
\end{proposition}
The performance of the algorithm is then determined by the sum of a \emph{bias term} -- characterizing how fast the initial condition $\theta_0$ is forgotten and which is increasing with $\norm{\theta_0-\theta^*}$; and a \emph{variance term} -- characterizing the effect of the noise in the gradient estimates
and that increases with the variance of the noise $\sigma^2$. Here the bias converges exponentially fast whereas the variance
is $O(\gamma)$. Note that the bias decrease is of the form $(1 - \gamma \mu)^n \delta_0$, which means that the typical number of iterations to reach stationarity is  $\Theta(\gamma^{-1})$.

As noted by \citet{bottou2018optimization}, this decomposition naturally leads to the question: which convergence rate can we hope getting if we keep a large step size as long as progress is being made but decrease it as soon as the iterates saturate? More explicitly, starting from $\theta_0$, one could run SGD with a constant step size $\gamma_0$ for $\Delta n_1$ steps until progress stalls. Then for $n \geq \Delta n_1$, a smaller step size $\gamma_1 = r\gamma_0$ (where $r \in (0, 1)$) is used in order to decrease the variance and therefore get closer to $\theta^*$ and so on. This simple strategy is implemented in \cref{alg:gen}. 
However the crucial difficulty here lies in detecting the saturation. Indeed when running SGD we do not have access to $\norm{\theta_n-\theta^*}$ and we cannot evaluate the successive function values $f(\theta_n)$ because of their prohibitively expensive cost to estimate. Hence, we focus on finding a statistical diagnostic which is computationally cheap and that gives an accurate restart time corresponding to saturation. 

\paragraph{Oracle diagnostic.}
Following this idea, assume first we have access to all the parameters of the problem: $\norm{\theta_0-\theta^*}$, $\mu$, $L$, $\sigma^2$.
Then reaching saturation translates into the bias term and the variance term from \cref{eq:needell} being of the same magnitude, i.e.
\[
(1 -  \gamma_0 \mu)^{\Delta n_{1}} \norm{\theta_{0} - \theta^*}^2 =  \frac{2\gamma_0 \sigma^2}{\mu}.
\]
This oracle diagnostic is formalized in \cref{alg:or}. The following proposition guarantees its performance.

\begin{algorithm}[t]
\caption{Convergence-Diagnostic algorithm}
\label{alg:gen}
\begin{algorithmic} 
\STATE \textbf{Input: } Starting point $\theta_0$, Step size $\gamma_0 > 0$, Step-size decrease $r \in (0, 1)$ \\
\STATE \textbf{Output:} $\theta_N$
\STATE $\gamma \gets \gamma_0$
\FOR{$n = 1$ to $N$}  
\STATE $\theta_n \gets \theta_{n-1} - \gamma f'_n(\theta_{n-1})$
\IF{\{ Saturation Diagnostic \} is True}
\STATE $\gamma \gets r \times \gamma$
\ENDIF
\ENDFOR
\STATE \textbf{Return: } $\theta_{N}$
\end{algorithmic}  
\end{algorithm}

\begin{algorithm}[t]
\caption{Oracle diagnostic}
\label{alg:or}
\begin{algorithmic}
\STATE \textbf{Input: }  $\gamma$, $\delta_0$, $\mu$, $L$, $\sigma^2$, $n$
\STATE \textbf{Output:} Diagnostic boolean
\STATE Bias $\gets (1 -  \gamma \mu)^n \delta_0$ 
\STATE Variance $\gets \frac{2\gamma \sigma^2}{\mu}$
\STATE \textbf{Return: } \{ Bias $<$ Variance \}
\end{algorithmic}
\end{algorithm}

\begin{proposition}
\label{eq:oracle_rate}
Under \cref{as:unbiased_gradients,as:smoothness,as:sc,as:noise_condition}, consider \Cref{alg:gen} instantiated  with \Cref{alg:or} and parameter   $r \in (0, 1)$ . Let $\gamma_0 \in (0, 1 / 2 L)$, $\delta_0 = \norm{\theta_0 - \theta^*}^2$ and $\Delta n_1=\frac{1}{\gamma_0\mu}\log (\frac{\mu \delta_0}{2\gamma_0 \sigma^2})$.
Then, we have for all $n \leq \Delta n_{1}$:
\[
\E{\norm{\theta_{n} - \theta^*}^2} \leq  (1-\gamma_0\mu)^n \delta_0 +\frac{2\gamma_0\sigma^2}{\mu},
\]
and for all $n > \Delta n_{1}$:
\[
\E{\norm{\theta_{n} - \theta^*}^2} \leq  \frac{8 \sigma^2}{ \mu^2 (n-\Delta n_{1})(1 - r)} \ln{\Big( \frac{2}{ r }\Big)}.\]
\end{proposition}
The proof of this Proposition is given in~\Cref{appsec:proofoforacle_rate}. We make the following observations:
\begin{itemize}[leftmargin=*,itemsep=0pt, topsep=0pt]
\item The rate  $O(1 / \mu^2 n)$ is optimal for last-iterate convergence for strongly-convex problem~\cite{nguyen2019tight} and is also obtained by  SGD with decreasing step size $\gamma_n = C / \mu n$ where $C > 2$~\cite{moulines2011non}.
More generally, the rate $O(1/n)$ is known to be information-theoretically optimal for strongly-convex stochastic approximation~\cite{NemYud83}. 
\item %
To reach an $\varepsilon$-optimal point,  $O\big(\frac{\sigma^2}{\mu^2\varepsilon}+\frac{L}{\mu}\log(\frac{\mu L \delta_0}{\sigma^2})\big)$ calls to the gradient oracle are needed. Therefore the bias is forgotten exponentially fast. This stands in sharp contrast to averaged SGD for which there is no exponential forgetting of initial conditions~\cite{moulines2011non}. 
\item We present in~\Cref{appsec:unif_convex} additional results for weakly and uniformly convex functions. In this case too, the oracle diagnostic-based algorithm recovers the optimal rates of convergence. However these results hold only for the restart iterations $n_k$, and the behaviour in between each can be theoretically arbitrarily bad. 
    \item Our algorithm shares key similarities with the algorithm of \citet{hazan2014optimal} which halves the learning rate every $2^k$ iterations but with the different aim of obtaining the sharp $O(1/n)$ rate in the non-smooth setting.
\end{itemize}

 This strategy is called oracle since all the parameters must be known and, in that sense, \cref{alg:or} is clearly  non practical. However \cref{eq:oracle_rate} shows that \cref{alg:gen} implemented with a practical and suitable diagnostic is \textit{a priori} a good idea since it leads to the optimal rate $O(1 / \mu^2 n)$ without having to know the strong convexity parameter $\mu$ and the rate $\alpha$ of decrease of the step-size sequence $\gamma_n=O(n^{-\alpha})$. The aim of the following sections is  to propose a computationally cheap and efficient statistic that detects the transition between transience and stationarity.

\section{Pflug's Statistical Test for stationarity  }
\label{sec:pflug}

In this section we analyse a statistical diagnostic first developed by \citet{pflug1983determination} which relies on the sign of the inner product of two consecutive stochastic gradients $\ps{f'_{k+1}(\theta_k)}{f'_{k + 2}(\theta_{k+1})}$. Though this procedure was developed several decades ago, no theoretical analysis had been proposed yet 
despite the fact that  several  papers have recently showed renewed interest in it~\cite{chee2018convergence,lang2019using,sordello2019robust}. Here we  show that whilst it is true this statistic becomes in expectation negative at stationarity, it is provably inefficient  to properly detect the restart time -- for the particular example of quadratic functions.

\subsection{Control of the expectation of Pflug's statistic}
The general motivation behind Pflug's statistic is that during the transient phase the inner product is in expectation positive and during the stationary phase, it is in expectation negative. Indeed, in the transient phase, where  $\norm{\theta-\theta^*}>>\sqrt{\gamma} \sigma$, the effect of the noise is negligible and the behavior of the iterates is very similar to the one of  noiseless gradient descent (i.e,  $\varepsilon(\theta)=0$ for all $\theta\in\R^d$) which satisfies: 
\[
\langle f'(\theta),f'(\theta-\gamma f'(\theta)) \rangle = \norm{f'(\theta)}^2+O(\gamma) >0.
\]
On the other hand, in the stationary phase, we may intuitively assume starting from $\theta_0=\theta^*$ to obtain
\begin{align*}
\!\E{\ps{f'_{1}(\theta_0)}{\!f'_{2}(\theta_1}} \!&=\! - \E{\ps{\varepsilon_1}{f'(\theta^* + \gamma \varepsilon_1)} } \\ & \!=\! - \gamma \tr f''(\theta^*) \E{\varepsilon_1\varepsilon_1^\top} \!+\!O(\gamma)
< 0.
\end{align*}

The single values $ \ps{f'_{k+1}(\theta_k)}{f'_{k + 2}(\theta_{k+1})}$ are too noisy, which leads \cite{pflug1983determination} in considering the running average: 
\[
S_n =  \frac{1}{n} \sum_{k= 0}^{n-1} \ps{f'_{k+1}(\theta_k)}{f'_{k + 2}(\theta_{k+1})}.
\]
This average can easily be computed online with negligible extra computational and memory costs.  \citet{pflug1983determination} then advocates to decrease the step size when the statistic becomes negative, as explained in \cref{alg:gen}. A burn-in delay $n_b$ can also be waited to avoid the first noisy values. 
\begin{algorithm}[H]
  \caption{Pflug's diagnostic}
  \label{alg:pflug}
   \begin{algorithmic} 
   \STATE \textbf{Input:}  $(f'_k(\theta_{k-1}))_{0 \leq k \leq n}$, $n_b > 0$ 
   \STATE \textbf{Output:} Diagnostic boolean
  \STATE $S \gets 0$
  \FOR{$k = 2$ to $n$}  
  \STATE $S \gets S + \ps{f'_k(\theta_{k-1})}{f'_{k-1}(\theta_{k-2})}$
  \ENDFOR
  \STATE \textbf{Return :} $\{ S < 0 \} \ \text{AND} \ \{ n > n_b \}$
  \end{algorithmic}  
\end{algorithm}
For quadratic functions, \citet{pflug1988adaptive} first shows that, when $\theta\sim \pigamma$ at stationarity, the inner product of two successive stochastic gradients is  negative in expectation. %
To extend this result to the wider class of smooth strongly convex functions, we make the following technical assumptions.
\begin{assumption}[Five-times differentiability of $f$] 
\label{as:5_times_differentiability}
The function $f$ is five times continuously differentiable with second to fifth uniformly bounded derivatives.
\end{assumption}

\begin{assumption}[Differentiability of the noise] 
\label{as:supp_noise_conditions}
The noise covariance function $\mathcal{C}$ is three times continuously differentiable with locally-Lipschitz derivatives.
Moreover $\mathbb{E}(\norm{\varepsilon_1(\theta^*)}^6)$ is finite.
\end{assumption}
These assumptions 
are satisfied in natural settings.  The following proposition addresses the sign of the expectation of Pflug's statistic.
\begin{proposition}
\label{eq:inner_product_expansion}
Under \cref{as:unbiased_gradients,as:sc,as:smoothness,as:noise_condition,as:5_times_differentiability,as:supp_noise_conditions}, for $\gamma\in(0,1/2L)$ , let $\pi_{\gamma}$ be the unique stationary distribution. Let $\theta_1 = \theta_0 - \gamma f'_1(\theta_0)$. For any starting point $\theta_0$, we have
\begin{multline*}\!\E{  \ps{ f'_1(\theta_0)}{ \!f'_2(\theta_1)} } \geq  (1\!-\!\gamma L) \norm{f'(\theta_0)}^2  
\\-\!\gamma L \!\tr \!\mathcal{C}(\theta_0)\!+ \!O(\gamma^{2}).
\end{multline*}
And for $\theta_0 \sim \pigamma$, we have:
\begin{align*}
\mathbb{E}_{\pi_{\gamma}} \left [ \ps{ f'_1(\theta_0)}{\! f'_2(\theta_1)} \right ] = &- \frac{1}{2} \gamma \!\tr\!f''(\theta^*) \C(\theta^*)+ O(\gamma^{3/2}).
\end{align*}
\end{proposition}
\begin{sproof} The complete proof is given in~\Cref{appsec:proof_pflug_expectation}. The first part relies on a simple Taylor expansion  of $f'$ around $\theta_0$. For the second part, we decompose:
\begin{align*}
\mathbb{E} [  \ps{ f'_1(\theta_0)}{& f'_2(\theta_1)} \  | \ \theta_0 ] = \\ & \underbrace{\E{\langle f'(\theta_0), f'(\theta_1) \rangle  \  | \ \theta_0  }}_{S_{\text{grad}}}+ \underbrace{\E{\langle \varepsilon_1, f'(\theta_1)  \  | \ \theta_0  \rangle }}_{S_{\text{noise}}}.
\end{align*}
Then, applying successive Taylor expansions of $f'$ around the optimum $\theta^*$ yields for both terms:
\begin{align*}
&S_{\text{grad}} =  \tr{f''(\theta^*)^2  (\theta_0 - \theta^*)^{\otimes 2}}   + O(\gamma^{3 / 2}),\\
&S_{\text{noise}} = - \gamma \tr{f''(\theta^*) \mathcal{C}(\theta_0)} + O(\gamma^{3 / 2}).
\end{align*}
Using results from \citet{dieuleveut2017bridging} on $\Epi{(\theta_0 - \theta^*)^{\otimes 2}}$ and $\Epi{\mathcal{C}(\theta_0)}$ then leads to
\begin{align*}
&\Epi{S_{\text{grad}}}= \frac{1}{2}\gamma \tr{f''(\theta^*) \mathcal{C}(\theta^*)}   + O(\gamma^{3 / 2}),\\
&\Epi{S_{\text{noise}}} = - \gamma \tr{f''(\theta^*) \mathcal{C}(\theta^*)} +  O(\gamma^{3 / 2}). \tag*{\qedhere}
\end{align*}
\end{sproof}
 We note that, counter intuitively, the inner product is not negative because the iterates bounce around $\theta^*$ (we still have $S_{\text{grad}}=\E{\ps{f'(\theta_1)}{f'(\theta_0)}} >0$), but because the noise part $S_{\text{noise}}=\E{\ps{\varepsilon_1}{f'(\theta_1)}} $ is negative and dominates the gradient part $S_{\text{grad}}$.

In the case where $f$ is quadratic we immediately recover the result of \citet{pflug1988stepsize}. We note that  \citet{chee2018convergence} show a similar result but under far more restrictive assumptions on the noise distribution and the step size.

\cref{eq:inner_product_expansion} establishes that  the  sign  of  the expectation of the inner product between two consecutive gradients characterizes the transient and stationary regimes:
for an iterate $\theta_0$ far away from the optimum, i.e. such that $\norm{\theta_0-\theta^*}$ is large, the expected value of the statistic is positive whereas it becomes negative when the iterates reach stationarity. This makes clear the motivation of considering the sign of the inner products as a convergence diagnostic. Unfortunately this result does not guarantee the good performance of this statistic. Even though the inner product is negative, its value is only $O(\gamma)$. It is then difficult to distinguish $\ps{f'_{k+1}}{f'_{k + 2}}$ from zero for small step size $\gamma$. In fact,  we now show that even for simple quadratic functions, the statistical test is unable to offer an adequate convergence diagnostic.

\subsection{Failure of Pflug's method for Quadratic Functions}
In this section we show that Pflug's diagnostic fails to accurately detect convergence, even in the simple framework of  quadratic objective functions with additive noise.  While we have demonstrated in \Cref{eq:inner_product_expansion} that the sign of its expectation characterizes the transient and stationary regime, we show that the running average $S_n$ does not concentrate enough around its mean to result in a valid test. Intuitively, from a restart when we leave stationarity: (1) the  expectation is positive but smaller than $\gamma$ , and (2) the standard deviation of $S_n$ is not decaying with $\gamma$, but only with the number of steps over which we average, as $1/\sqrt{n}$. As a consequence, in order to ensure that the sign of $S_n$ is the same as the sign of its expectation, we would need to average over more than $1/\gamma^2$ steps, which is orders of magnitude bigger than the optimal restart time of $\Theta(1/\gamma)$ (See \cref{sec:BiasVariance}).
We make this statement quantitative under simple assumptions on the noise.
\begin{assumption}[Quadratic semi-stochastic setting] 
\label{as:quadratic_f} There exists a symmetric positive semi-definite matrix $H$ such that $f(\theta) = \frac{1}{2} \theta^T H \theta$. The noise $\varepsilon_i(\theta) = \xi_i$ is independent of $\theta$ and:
\vspace{-0.2cm}
\begin{equation*}
    (\xi_i)_{i \geq 0} \text{ are i.i.d. }, \ \E{\xi_i} = 0, \ \E{\xi_i^T \xi_i} = C.
\end{equation*}
\vspace{-0.4cm}
\end{assumption}
In addition we make a simple assumption on the noise:
\begin{assumption}[Noise symmetry and continuity] 
\label{as:noise_symmetry} 
The function $\Prob{\xi_1^T \xi_2 \geq x}$ is continuous in $x=0$ and 
\vspace{-0.1cm}
\begin{equation*}
\Prob{\xi_1^T \xi_2 \geq x} = \Prob{\xi_1^T \xi_2 \leq -x} \ \ \text{ for all } x\geq 0.
\end{equation*}
\end{assumption}
\vspace{-0.2cm}
This assumption is made for ease of presentation and can be relaxed.  We make use of the following notations. We assume SGD is run with a constant step size $\gamma_{old}$ until the stationary distribution $\pi_{\gamma_{old}}$ is reached. The step size is then decreased and SGD is run with a smaller step $\gamma \!=\! r \!\times\! \gamma_{old}$. Hence the iterates cease to be at stationarity under  $\pi_{\gamma_{old}}$ and start a transient phase towards $\pi_{\gamma}$. We denote by $\mathbb{E}_{\theta_0 \sim \gamma_{old}}$ (resp. $\mathbb{P}_{\theta_0 \sim \gamma_{old}}$) the expectation (resp. probability) of a random variable (resp. event) when the initial $\theta_0$ is sampled from the old  distribution $\pi_{\gamma_{old}}$ and a new step size $\gamma\! =\! r \!\times\! \gamma_{old}$ is used.  Note that $\mathbb{E}_{\theta_0 \sim \gamma_{old}}$ and $\mathbb{E}_{\pigamma}$ have different meanings, the latter being the expectation under $\pigamma$.

We first split $S_n$ in  a $\gamma$-dependent and a $\gamma$-independent part.
\begin{lemma}{\label{eq:var_R}}
Under \cref{as:quadratic_f}, let $\theta_0 \sim \pi_{\gamma_{old}}$ and assume we run SGD with a smaller step size $\gamma = r \times \gamma_{old}$, $r \in (0,1)$. Then, the statistic $S_n$ can be decomposed as:
$
S_n = - R_{n, \gamma} +  \chi_{n}
$. The part $\chi_{n}$ is independent of $\gamma$ and 
\begin{align*}
   & \Erestart{R_{n, \gamma}^2} \leq M (\frac{\gamma}{n} + \gamma^2);\\
&\E{\chi_n}=0 \ ,  \Var{\chi_n} = \frac{1}{n} \tr{(C^2)} \ \text{ and } \\
&\Var{\chi_n^2} = \frac{\E{(\xi_1^T \xi_2)^4} - \trsq{C^2}}{n^3},
\end{align*}

where $M$ is independent of $\gamma$ and $n$.
\end{lemma}

Thus the variance of $\chi_n$ does not depend on $\gamma$ while, from a restart, the second moment  $\Erestart{R_{n, \gamma}^2}$ is $O(\frac{\gamma}{n} + \gamma^2)$. Therefore the signal to noise ratio is high. This property is the main idea behind the proof of the following proposition.

\begin{proposition}
\label{eq:pflug_probality}
Under \cref{as:quadratic_f,as:noise_symmetry}, let $\theta_{0}  \sim \pi_{\gamma_{old}}$ and run SGD with $\gamma = r \times \gamma_{old}$, $r \in (0, 1)$. Then for all $0 \leq \alpha < 2$ , and $n_{\gamma} = O(\gamma^{-\alpha})$ we have:
\[
\lim_{\gamma \to 0}\ProbRest{S_{n_{\gamma}} \leq 0} =\frac{1}{2}.
\]
\end{proposition}
\begin{sproof}
The complete proofs of \cref{eq:var_R,eq:pflug_probality} are given in~\Cref{appsec:proof_pflug_probality}. The main idea is that the signal to noise ratio is too high. The signal during the transient phase is positive and $O(\gamma)$. However the variance of $S_n$ is $O(1 / n)$. Hence $\Omega (1 / \gamma^2)$ iterations are typically needed in order to have a clean signal. Before this threshold, $S_n$ resembles a random walk and its sign gives no information on whether saturation is reached or not, this leads to early on restarts. 
\end{sproof}
We make the following observations.
\begin{itemize}[itemsep=0pt, topsep=0pt,leftmargin=*]
 \item Note that the typical time to reach saturation with a constant step size $\gamma$ is of order $1 / \gamma$ (see \cref{sec:BiasVariance}). 
 We should expect Pflug's statistic to satisfy $\lim_{\gamma \to 0} \ProbRest{S_{n_{b}} \leq 0} = 0$ for all constant burn-in time $n_b$ smaller than the typical saturation time $O(1/\gamma)$ -- since the statistic should not detect saturation before it is actually reached. \cref{eq:pflug_probality} shows that this is not the case and that the step size is therefore decreased too early. This phenomenon is clearly seen in \cref{fig:pflug_abusive_restarts} in \cref{sec:experiments}.

    \item We note that \citet{pflug1988adaptive} describes an opposite result. We believe this is due to a miscalculation of $\Var {\chi_n}$ in his proof (see detail in~\Cref{appsec:proof_pflug_mistake}).
    
     \item \citet{lang2019using} similarly point out the existence of a large variance in the diagnostic proposed by \citet{yaida18}. They make the strategy more robust by implementing a formal statistical test, to only reduce the learning rate when the limit distribution has been reached with \emph{high confidence}. Unfortunately, \cref{eq:pflug_probality} entails that more than $O(1/\gamma^2)$ iterations are needed to accurately detect convergence for Pflug's statistic, and we thus believe that Lang's approach would be too conservative and would not reduce the learning rate often enough.
\end{itemize}

Hence Pflug's diagnostic is inadequate and leads to poor experimental results (see \cref{sec:experiments}). We propose then a novel simple distance-based diagnostic which enjoys state-of-the art rates for a variety of classes of convex functions.

\section{A new distance-based statistic}
\label{sec:new_stat}
We propose here a very simple statistic based on the distance between the current iterate $\theta_n$ and the iterate from which the step size has been last decreased. Indeed, we would ideally like to decrease the step size when $\norm{\eta_n} \!= \norm{\theta_n \!-\! \theta^*}$  starts to saturate. Since the optimum $\theta^*$ is not known, we cannot track the evolution of this criterion. However it has a similar behaviour as $\norm{\Omega_n} \!=\! \norm{\theta_n\! - \theta_0}$, which we can compute. 
This is seen through the simple equation
$$\norm{\Omega_n}^2 = \norm{\eta_n}^2 + \norm{\eta_0}^2 - 2 \ps{\eta_n}{\eta_0}.$$
The value $\norm{\eta_n}^2$ is then expected to saturate roughly at the same time as $\norm{\Omega_n}^2$. In addition, $\norm{\theta_n - \theta_0}^2$ describes a large range of values which can be easily tracked, starting at $0$ and roughly finishing around $\norm{\theta^* - \theta_0}^2 + O(\gamma)$ (see \cref{corr:slopes}). It is worth noting this would not be the case if a different referent point, $\Tilde{\theta} \neq \theta_0$, was considered.

 To find a heuristic to detect the convergence of $\norm{\theta_n - \theta_0}^2$, we consider the particular setting of a quadratic objective with additive noise stated in \cref{as:quadratic_f}. In this framework we can compute the evolution of $\E{\norm{\Omega_{n}}^2}$ in closed-form .
\begin{proposition}
\label{eq:omega}
Let $\theta_0 \in \R^d$ and $\gamma\in (0,1/L)$. Let $\Omega_n = \theta_n - \theta_0$. Under \cref{as:quadratic_f} we have that:
\begin{align*}
\E{\norm{\Omega_{n}}^2}& =  \eta_0^T [I - (I - \gamma H)^n]^2 \eta_0 \\
& + \gamma \tr{[I - (I - \gamma H)^{2 n}] (2 I - \gamma H)^{-1} H^{-1} C}.
\end{align*}
\end{proposition}

The proof of this result is given in~\Cref{appsec:new_stat}. We can analyse this proposition in two different settings: for small values of $n$ at the beginning of the process and when the iterates $\theta_n$ have reached stationarity.
\begin{corollary}\label{corr:slopes}
Let $\theta_0 \in \R^d$ and $\gamma\in[0,1/L]$. Let $\Omega_n = \theta_n - \theta_0$. Under \cref{as:quadratic_f} we have that for all $n\geq0$:
\begin{align*}
\Epi{\norm{\Omega_{n}}^2} &= \norm{\eta_0}^2 + \gamma \tr{H^{-1} C (2I - \gamma H)^{-1}},\\
\E{\norm{\Omega_{n}}^2} &= \gamma^2 \eta_0^T H^2 \eta_0 \times n^2 + \gamma^2 \tr{C} \times n \\
& \qquad \qquad \qquad \qquad \ + o((n \gamma)^2).
\end{align*}
\end{corollary}
From \cref{corr:slopes} we have shown the following asymptotic behaviours:
\begin{itemize}[leftmargin=*,itemsep=0pt, topsep=0pt]
    \item \textit{Transient phase.} For $n \ll 1 / (\gamma L)$, in a log-log plot $\E{\norm{\Omega_{n}}^2}$ has a slope bigger than $1$.
    \item \textit{Stationary phase.} For $n \gg 1 / (\gamma \mu)$, $\E{\norm{\Omega_{n}}^2}$ is constant and therefore has a slope of $0$ in a log-log plot.
\end{itemize}{}

This dichotomy naturally leads  to a distance-based convergence diagnostic where the step size is decreased by a factor $1/r$ when the slope becomes smaller than a certain threshold smaller than $2$. The slope is computed between iterations of the form $q^k$ and $q^{k+1}$ for $q > 1$ and $k \geq k_0$.  
The method is formally described in  \cref{alg:omega}.
We impose a burn-in time $q^{k_0}$ in order to avoid unwanted and possibly harmful restarts during the very first iterations of the SGD recursion, it is typically worth $\sim 8$ ($q=1.5$ and $k_0 = 5$) in all our experiments, see \cref{sec:experiments} and \cref{appsec:ours_expe}. Furthermore note that from \cref{eq:needell}, saturation is reached at iteration $\Theta(\gamma^{-1})$. Therefore when the step-size is decreased as $\gamma \gets r \times \gamma$ then the duration of the transience phase is increased by a factor $1 / r$. This shows that it is sufficient to run the diagnostic every $q^k$ where $q$ is smaller than $1 / r$.  

\begin{algorithm}[H]
\caption{Distance-based diagnostic}
\label{alg:omega}
\begin{algorithmic} 
\STATE\textbf{Input:} $\theta_0$, $\theta_n$, $\theta_{n / q}$, $n$, $q > 1$, $k_0 \in \mathbb{N}^*$, $\text{thresh} \in (0, 2]$\\
\STATE \textbf{Output:} Diagnostic boolean
\IF{ $n = q^{k+1}$ for a $k \geq k_0$ in $\mathbb{N}^*$}
\STATE $S \gets \frac{\log{\norm{\theta_n - \theta_{0}}^2} - \log{\norm{\theta_{n / q} - \theta_{0}}^2}   }{\log{n} - \log{n/q}}$
\STATE \textbf{Return:}  $\{ S < \text{thresh} \}$
\ELSE{}
\STATE \textbf{Return: } \textbf{False}
\ENDIF
\end{algorithmic}  
\end{algorithm}
\vspace{-1em}

\section{Experiments}
\label{sec:experiments}
In this section, we illustrate our theoretical results with synthetic and real examples. We provide additional experiments in \Cref{appsec:ours_expe}. %
\paragraph{Least-squares regression.} We consider the objective $f(\theta) = \frac{1}{2 } \E{ (y_i - \ps{x_i}{\theta})^2}$. The inputs $x_i$ are i.i.d.~from $\mathcal{N}(0, H)$ where $H$ has random eigenvectors and eigenvalues $(1 / k)_{1\leq k \leq d}$. We note $R^2 = \tr{H}$. The outputs $y_i$ are generated following 
$y_i = \ps{x_i}{\theta^*} + \varepsilon_i$ 
where $(\varepsilon_i)_{1 \leq i \leq n}$ are i.i.d.~from $\mathcal{N}(0, \sigma^2)$.  We use averaged-SGD with constant step size $\gamma=1/2R^2$  as a baseline since it enjoys the optimal statistical rate $O(\sigma^2 d / n)$~\cite{bach2013non}.

\begin{figure}
    \begin{center}
    \includegraphics[trim = {3.6cm 0.1cm 3.7cm 0.5cm}, clip, width=\linewidth]{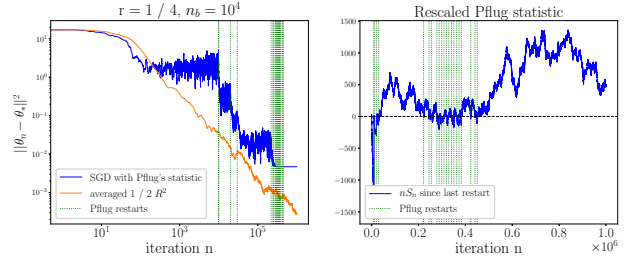}
    \end{center}
    \vspace{-0.3cm}
    \caption{Least-squares on synthetic data.  Left: least-squares regression. Right: Scaled Pflug's statistic $nS_n$. The dashed vertical lines correspond to Pflug's restarts. Note that only the left plot is in log-log scale.}
    \label{fig:pflug_abusive_restarts}
    \vspace{-1em}
    \end{figure}

\paragraph{Logistic regression setting.} We consider the objective $f(\theta) = \E{\log ( 1 + e^{- y_i \ps{x_i}{\theta}}}$. The inputs $x_i$ are generated the same way as in the least-square setting. The outputs $y_i \in \{ -1, 1 \}$ are generated following the logistic probabilistic model.
We use averaged-SGD with step-sizes $\gamma_n=1/ \sqrt{n}$  as a baseline since it enjoys the optimal rate $O(1/ n)$~\cite{bach2014adaptivity}. We also compare to online-Newton~\cite{bach2013non} which achieves better performance in practice.

\paragraph{ResNet18.} We train an 18-layer ResNet model \cite{he2016deep} on the CIFAR-10 dataset~\cite{Krizhevsky09learningmultiple} using SGD with a momentum of $0.9$, weight decay of $0.0001$ and batch size of $128$. To adapt the distance-based step-size statistic to this scenario, we use  Pytorch's \textit{ReduceLROnPlateau}() scheduler, created to detect saturation of arbitrary quantities. We use it to reduce the learning rate by a factor $r = 0.1$ when it detects that $\norm{\theta_n - \theta_{restart}}^2$ has stopped increasing. 
The parameters of the scheduler are set to: $\text{patience} = 1000$, $\text{threshold} = 0.01$. 
Investigating if this choice of parameters is robust to different problems and architectures would be a fruitful avenue for future research.
We compare our method to different step-size sequences where the step size is decreased by a factor $r = 0.1$ at various epoch milestones. Such sequences achieve state-of-the-art performances when the decay milestones are properly tuned. All initial step sizes are set to $0.1$.

\paragraph{Inefficiency of Pflug's statistic.} In order to test Pflug's diagnostic we consider the least-squares setting with $n = 1\text{e}6$, $d = 20$, $\sigma^2 = 1$.  \cref{alg:pflug} is implemented with a conservative burn-in time of $n_b = 1\text{e}4$ and \cref{alg:gen} with a discount factor $r = 1 / 4$. %
 We note in~\cref{fig:pflug_abusive_restarts} that the algorithm is restarted too often and abusively. This leads to small step sizes early on and to insignificant decrease of the loss afterward.  The signal of Pflug's statistic is very noisy, and its sign gives no significant information on weather saturation has been reached or not. As a consequence the final step-size is very close to 0.  We note that its behavior is alike the one of a random walk. On the contrary, averaged-SGD exhibits an $O(1/n)$ convergence rate. 
 We provide further experiments on Pflug's statistic in \Cref{appsec:pflug_experiments},  showing its systematic failure for several values of the decay parameter $r$, the seed and the burn-in.

\begin{figure}
\begin{center}
\includegraphics[width=0.95\linewidth]{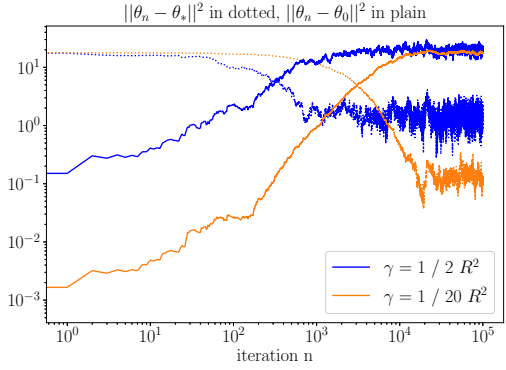}
\end{center}
\vspace{-0.3cm}
\caption{Logistic regression on synthetic dataset. $\norm{\theta_n - \theta^*}^2$ (dotted) and $\norm{\theta_n - \theta_0}^2$ (plain) for $2$ different step sizes. }
\label{fig:slopes_omega}
\vspace{-1em}
\end{figure}

\paragraph{Efficiency of the distance-based diagnostic.}
In order to illustrate the benefit of the distance-based diagnostic, we 
performed extensive experiments in several settings, more precisely: (1) Least Squares regression on a synthetic dataset, (2) Logistic regression on both synthetic and real data, (3) Uniformly convex functions, (4) SVM, (5) Lasso.
\emph{In all these settings, without any tuning, we achieve the same performance as the best suited method for the problem.} These experiments are detailed in~\Cref{appsec:ours_expe}. We hereafter present results for Logistic Regression.

First, we consider the logistic regression setting with $n = 1\text{e}5$, $d = 20$.  In \Cref{fig:slopes_omega}, we compare the behaviour of $\norm{\theta_n - \theta_0}^2$ and $\norm{\theta_n - \theta^*}^2$  for two different step sizes $1/2R^2$ and $1/20R^2$.  We first note that these two quantities have the same general behavior:  $\norm{\theta_n - \theta_0}^2$ stops increasing when $\norm{\theta_n - \theta^*}^2$ 
starts to saturate, and that this observation is consistent for the two step sizes. We additionally note that the average slope of $\norm{\theta_n - \theta_0}^2$ is of value $2$ during the transient phase and of value $0$ when stationarity has been reached. This demonstrates that, even if this diagnostic is inspired by the quadratic case, the main conclusions of \cref{corr:slopes} still hold for convex non-quadratic function and the distance-based diagnostic in \cref{alg:omega} should be more generally valid. We also notice that the two oracle restart times are spaced by $\log(20 / 2) = 1$ which confirms that the transient phase lasts $\Theta(1/\gamma)$.

\begin{figure}
\begin{center}
\includegraphics[trim = {2cm 0cm 4cm 0cm}, clip, width=\linewidth]{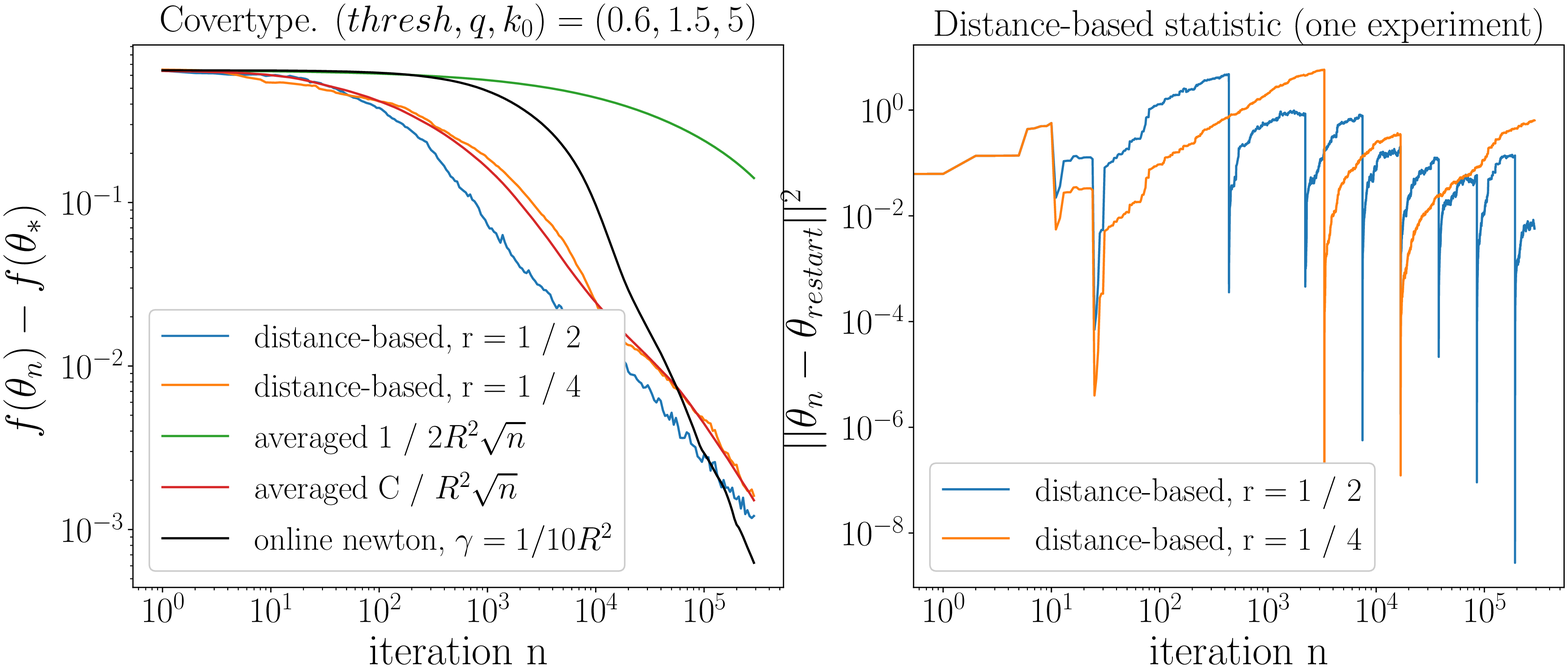}
\includegraphics[trim = {2cm 0cm 4cm 0cm}, clip, width=\linewidth]{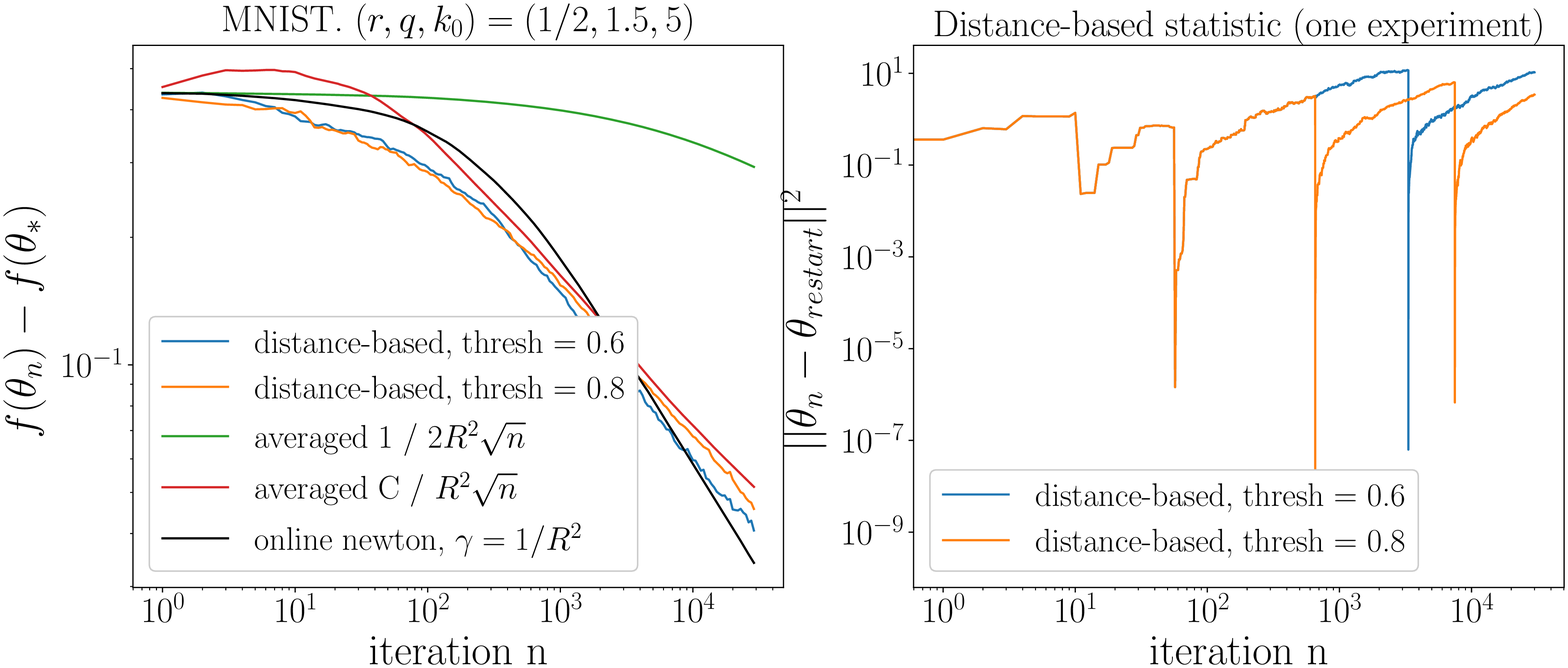}
\end{center}
\caption{Top: Covertype dataset. Two different values of $r$ are used: $1 / 2$, $1/4$. Bottom: MNIST dataset. Two different values of $thresh$ are used: $0.6$, $0.8$. Left: Logistic regression. Right: distance-based statistics $\norm{\theta_n - \theta_{restart}}^2$. }
\label{fig:rates_omega}
\vspace{-1em}
\end{figure}

We further investigate the performance of the distance-based diagnostic on real-world datasets: the Covertype dataset and the MNIST dataset\footnote{Covertype dataset available at \href{https://archive.ics.uci.edu/ml/datasets/covertype}{archive.ics.uci.edu/ml/datasets/covertype} 
and MNIST at \href{http://yann.lecun.com/exdb/mnist/}{yann.lecun.com/exdb/mnist}.}.  Each dataset is divided in two equal parts, one for training and one for testing. We then sample without replacement and perform a total of one pass over all the training samples. The loss is computed on the test set. This procedure is replicated $10$ times and the results are averaged. For MNIST the task consists in classifying the parity of the labels which are $\{ 0, \dots, 9\}$.
We compare our algorithm to: online-Newton ($\gamma = 1 / 10 R^2 $ for the Covertype dataset and $\gamma = 1 / R^2$ for MNIST) and averaged-SGD  with step sizes $\gamma_n=1/2 R^2\sqrt{n}$ (the value suggested by  theory) and  $\gamma_n=C/\sqrt{n}$ (where the parameter $C$ is tuned to achieve the best testing error). In \cref{fig:rates_omega}, we present the results.
Top row corresponds to the Covertype dataset for two different values of the decrease coefficient $r=1/2$ and $r=1/4$, the other parameters are set to 
$(tresh, q, k_0) = (0.6, 1.5, 5)$, left are shown the convergence rates for the different algorithms and parameters, right are plotted the evolution of the distance-based statistic $\norm{\theta_n-\theta_0}^2$. %
Bottom row corresponds to the MNIST dataset for two different values of the threshold $thresh=0.6$ and $thresh=0.8$, the other parameters are set to 
$(r, q, k_0) = (1/2, 1.5, 5)$, left are shown the convergence rates for the different algorithms and parameters, right are plotted the evolution of the distance-based statistic $\norm{\theta_n-\theta_0}^2$.
The initial step size for our distance-based algorithm was set to $4 / R^2$. 
Our adaptive algorithm obtains comparable performance as online-Newton and optimally-tuned averaged SGD, enjoying a convergence rate $O(1/n)$, and better performance than theoretically-tuned averaged-SGD. Moreover we note that the convergence of the distance-based algorithm is the fastest early stage. Thus this algorithm seems to benefit from the same exponential-forgetting of initial conditions as the oracle diagnostic (see \cref{eq:oracle_rate}). We point out that our algorithm is relatively independent of the choice of $r$ and $thresh$. We also note (red and green curves) that the theoretically optimal step size is outperformed by the hand-tuned one with the same decay, which only confirms the need for adaptive methods. On the right is plotted the statistic during the SGD procedure. Unlike Pflug's one, the signal is very clean, which is mostly due to the large range of values that are taken.

\begin{figure}
\begin{center}
\includegraphics[trim = {3cm 0.2cm 4.5cm 2cm}, clip, width=\linewidth, height=3.4cm]{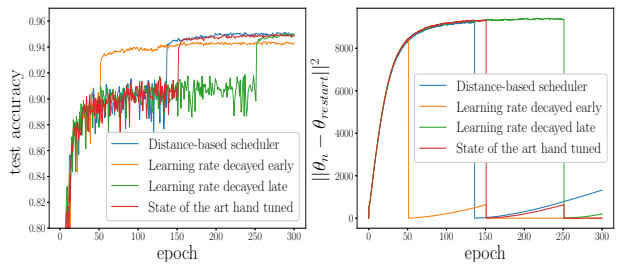}
\end{center}
\caption{ResNet18 trained on Cifar10. Left: test accuracies. Right: distance-based statistic $\norm{\theta_n - \theta_{restart}}^2$.}
\label{fig:deep_learning}
\vspace{-1em}
\end{figure}

\paragraph{Application to deep learning.}
We conclude by testing the distance-based statistic on a deep-learning problem in \cref{fig:deep_learning}. In practice, the learning rate is decreased when the accuracy has stopped increasing for a certain number of epochs. In red is plotted the accuracy curve obtained when the learning rate is decreased by a factor $r = 0.1$ at epochs $150$ and $250$. These specific epochs have been manually tuned to obtain state of the art performance.

Looking at the red accuracy curve, it seems natural to decrease the learning rate earlier around epoch $50$ when the test accuracy has stopped increasing. However doing so leads to a lower final accuracy (orange curve). On the other hand, decreasing the learning rate later, at epoch $250$, leads to a good final accuracy but takes longer to reach it. If instead of paying attention to the test accuracy we focus on the metric $\norm{\theta_n - \theta_{restart}}^2$ we notice that it still notably increases after epoch $50$ and until epoch $150$. This phenomenon manifests that this statistic
contains information that cannot be simply obtained from the test accuracy curve. Hence when the ReduceLROnPlateau scheduler is implemented using the distance-based strategy, the learning rate is automatically decreased around epoch $140$ and kept constant beyond (blue curve) which leads to a final state-of-the-art accuracy.

Therefore our distance-based statistic seems also to be a promising tool to adaptively set the step size for deep learning applications. We hope this will inspire further research.

\section*{Conclusion}

In this paper we studied convergence-diagnostic step-sizes. We first showed that such step-sizes make sense in the smooth and strongly convex framework since they recover the optimal $O(1 / n)$ rate with in addition an exponential decrease of the initial conditions. Two different convergence diagnostics are then analysed. First, we theoretically prove that Pflug's diagnostic leads to abusive restarts in the quadratic case. We then propose a novel diagnostic which relies on the distance of the final iterate to the restart point. We provide a simple restart criterion and theoretically motivate it in the quadratic case. The experimental results on synthetic and real world datasets show that our simple diagnostic leads to very satisfying convergence rates in a variety of frameworks. 

An interesting future direction to our work would be to theoretically prove that our diagnostic leads to adequate restarts, as seen experimentally. It would also be interesting to explore more in depth the applications of our diagnostic in the non-convex framework.

\section*{Acknowledgements}

The authors would like to thank the reviewers for useful suggestions as well as Jean-Baptiste Cordonnier for his help with the experiments.

\clearpage

\bibliography{bio}
\bibliographystyle{icml2020}

\onecolumn
\appendix

\section*{Organization of the Appendix}
In the appendix, we provide additional experiments and detailed proofs to all the results presented in the main paper. 
\begin{enumerate}
    \item In \cref{appsec:expe} we provide additional experiments. In \cref{appsec:pflug_experiments} we show that Pflug's diagnostic fails for different values of decrease factor $r$ and burn-in time $n_b$; together with 
     a simple experimental illustration of \cref{eq:pflug_probality}.
    Then in \cref{appsec:ours_expe} we investigate the performance of the distance-based statistic in different settings and for different values of $r$ and of the threshold value $thresh$. These settings are: Least-squares,
    Logistic regression,
    SVM,
    Lasso regression,
    and the Uniformly convex setting.
    \item In \cref{secapp:oracle} we prove \cref{eq:oracle_rate} as well as a similar result for uniformly convex functions.
    \item In \cref{appsec:pflugs} we prove \cref{eq:inner_product_expansion} and \cref{eq:pflug_probality} .
    \item Finally in \cref{appsec:new_stat} we prove \cref{eq:omega} and \cref{corr:slopes}.
\end{enumerate}

\section{Supplementary experiments}
\label{appsec:expe}
 Here we provide additional experiments for the Pflug diagnostic and the distance-based statistic in different settings.

 \begin{figure}[t]
\begin{center}
\includegraphics[trim = {1cm 0cm 1cm 1cm}, clip, width=\linewidth]{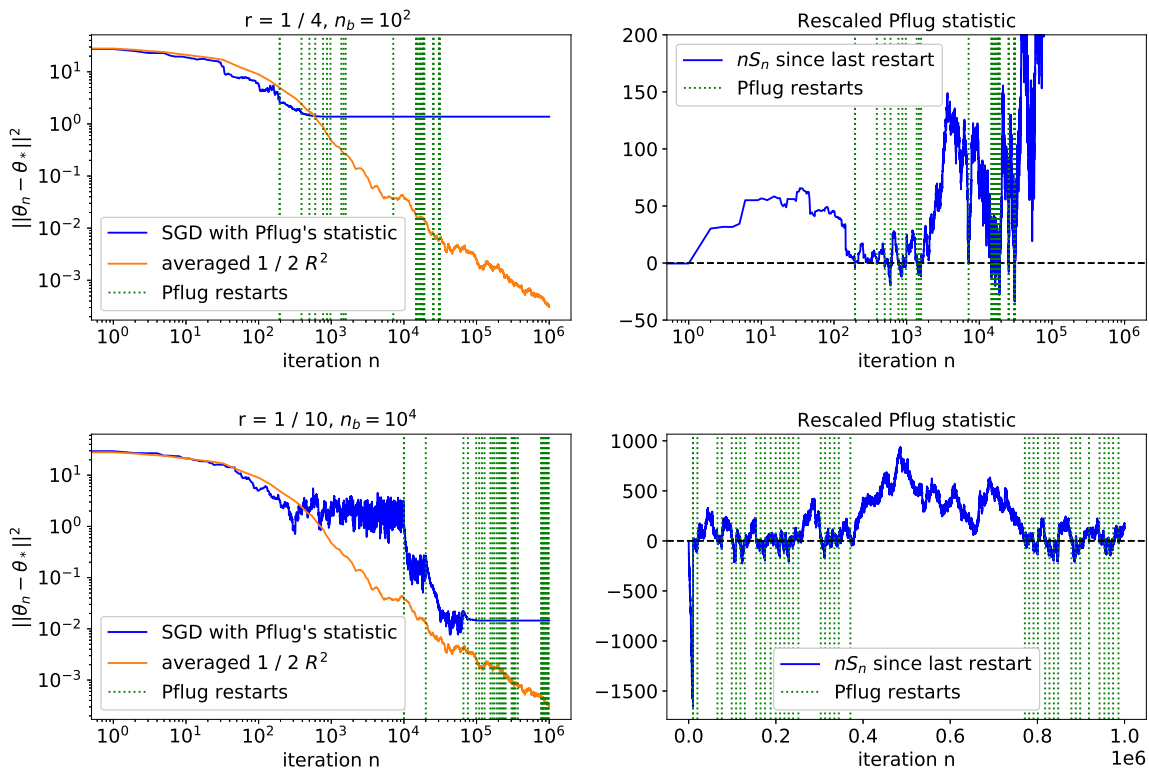}
\end{center}
\caption{Least-squares on synthetic data ($n = 1\text{e}6$, $d = 20$, $\sigma^2 = 1$). Left: least-squares regression. Right: Scaled Pflug statistic $n S_n$. The dashed vertical lines correspond to Pflug's restarts. Note that the x-axis of the bottom right plot is not in log scale.  Top parameters: $r = 1/10$, $n_b = 10^4$.  Bottom parameters: $r = 1/4$, $n_b = 10^2$. Initial learning rates set to $1 / 2 R^2$.}
\label{app_fig:pflug_abusive_restarts}
\end{figure}

\subsection{Supplementary experiments on Pflug's diagnostic}
\label{appsec:pflug_experiments}

We test Pflug's diagnostic in the least-squares setting with $n = 1\text{e}6$, $d = 20$, $\sigma^2 = 1$, $\gamma_0 = 1 / 2 R^2$.  Notice that as in \cref{fig:pflug_abusive_restarts}, Plug's diagnostic fails for different values of the algorithm's parameters. Indeed parameters $(r, n_b) = (1 / 4, 10^2)$ (\cref{app_fig:pflug_abusive_restarts} top row) and $(r, n_b) = (1 / 10, 10^4)$ (\cref{app_fig:pflug_abusive_restarts}  bottom row) both lead to abusive restarts (dotted vertical lines) that do not correspond to iterate saturation. These restarts lead to small step size too early and insignificant progress of the loss afterwards. Notice that in both cases the behaviour of the rescaled statistic $n S_n$ is similar to a random walk. On the contrary, as the theory suggests \citep{bach2013non} averaged-SGD exhibits a $O(1/n)$ convergence rate.
\begin{figure}[t]
\begin{center}
\includegraphics[trim = {5cm 0cm 5cm 0cm}, clip, width=\linewidth]{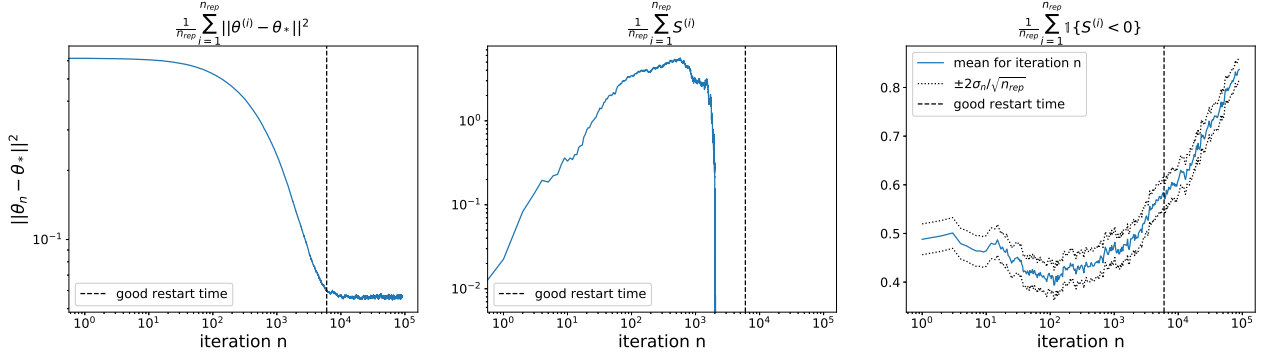}
\end{center}
\caption{Least-squares on synthetic data ($n = 1\text{e}5$, $d = 20$, $\sigma^2 = 1$). Parameters: $\gamma_{\text{old}} = 1 / 5 R^2$, $r = 1/10$, $n_{rep} = 10^3$. Left: least-squares regression averaged over all $n_{rep}$ samples. Middle: average of Pflug's statistic over all $n_{rep}$ samples. Right: fraction of runs where the statistic is negative at iteration $n$. The two dotted lines roughly correspond to the $95\%$ confidence intervals. }
\label{app_fig:illustr_theorem_pflug}
\end{figure}

In order to illustrate \cref{eq:pflug_probality} in the least-squares framework, we repeat $n_{rep}$ times the same experiment which consists in running constant step-size SGD from an initial point $\theta_0 \sim \pi_{\gamma_{\text{old}}}$ with a smaller step-size $\gamma = r \times \gamma_{\text{old}}$. The starting point $\theta_0 \sim \pi_{\gamma_{\text{old}}}$ is obtained by running for a sufficiently long time SGD with constant step size $\gamma_{\text{old}}$. In \cref{app_fig:illustr_theorem_pflug} we implement these multiple experiments with $n = 1\text{e}5$, $d = 20$, $\sigma^2 = 1$. In the left plot notice the two characteristic phases: the exponential decrease of $\norm{\theta_n - \theta^*}$ followed by the saturation of the iterates, the good restart time corresponding to this transition is indicated by the black dotted vertical line. Consistent with \cref{eq:inner_product_expansion}, we see in the middle plot that in expectation Pflug's statistic is positive then negative (the curve disappears as soon as its value is negative due to the plot in log-log scale). This change of sign occurs roughly at the same time as when the iterates saturate. However, in the right graph we plot for each iteration $k$ the fraction of runs for which the statistic $S_k$ is negative. We see that this fraction is close to $0.5$ for all $k$ smaller than the good restart time. Since for $n_{rep}$ big enough $\frac{1}{n_{rep}} \sum_{i = 1}^{n_{rep}} \mathbbm{1} \{ S_k^{(i)} < 0 \} \sim \mathbb{P} (S_k^{(i)} < 0 )$, this is an illustration of \cref{eq:pflug_probality}. Hence whatever the burn-in $n_b$ fixed by Pflug's algorithm, there is a chance out of two of restarting too early.

\subsection{Supplementary experiments on the distance-based diagnostic}
\label{appsec:ours_expe}

In this section we test our distance-based diagnostic in several settings. 
\begin{figure}[h!]
\begin{center}
\includegraphics[trim = {5cm 0cm 5cm 0cm}, clip, width=\linewidth, height = 10cm]{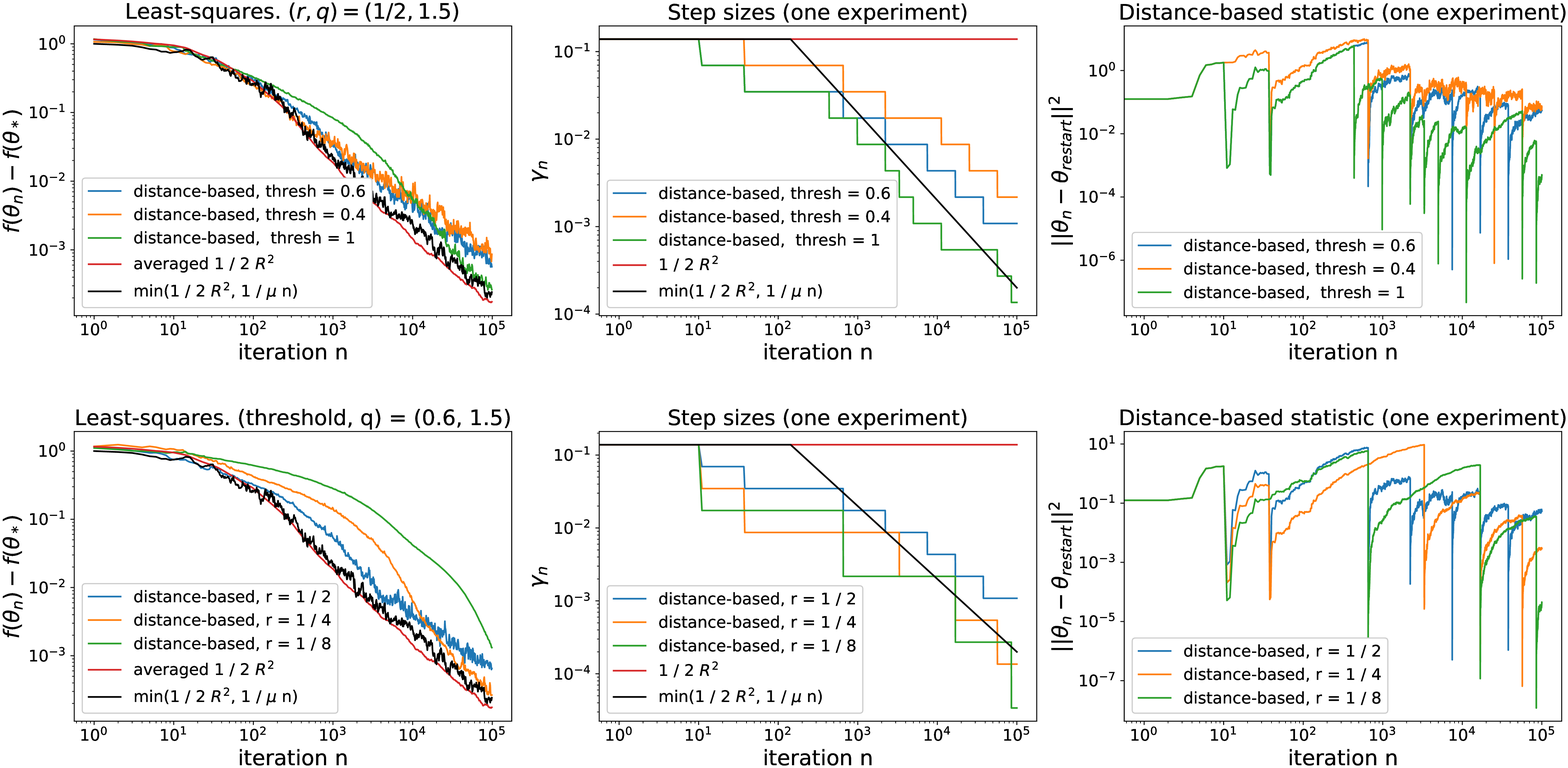}
\end{center}
\caption{Least-squares on synthetic data ($n = 1\text{e}5$, $d = 20$, $\sigma^2 = 1$). All initial step sizes of $1 / 2 R^2$. Top distanced-based parameters: $(r, q, k_0) = (1 / 2, 0.5, 5)$. Bottom distanced-based parameters: $(thresh, q, k_0) = (0.6, 1.5, 5)$. The losses on the left plot are averaged over $10$ replications.}
\label{app_fig:least_squares}
\end{figure}
\begin{figure}[h!]
\begin{center}
\includegraphics[trim = {5cm 0cm 5cm 0cm}, clip, width=\linewidth]{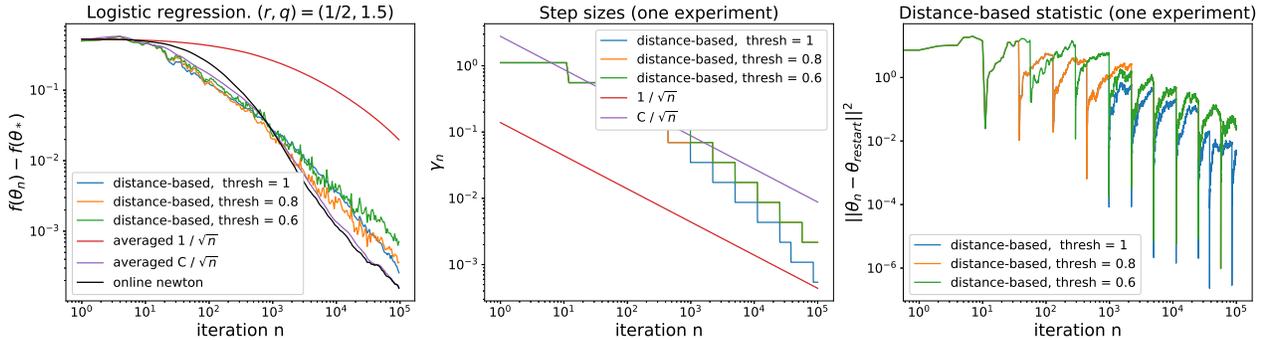}
\end{center}
\caption{Logistic regression on synthetic data ($n = 1\text{e}5$, $d = 20$). Distanced-based parameters: $(r, q, k_0) = (1 / 2, 1.5, 5)$ and $\gamma_0 = 4 / R^2$. The losses on the left plot are averaged over $10$ replications.}
\label{app_fig:logistic_regression}
\end{figure}
\paragraph{Least-squares regression.} We consider the objective $f(\theta) = \frac{1}{2 } \E{ (y - \ps{x}{\theta})^2}$. The inputs $x_i$ are i.i.d.~from $\mathcal{N}(0, H)$ where $H$ has random eigenvectors and eigenvalues $(1 / k)_{1\leq k \leq d}$. We note $R^2 = \tr{H}$. The outputs $y_i$ are generated following 
the generative model 
$y_i = \ps{x_i}{\theta^*} + \varepsilon_i$ 
where $(\varepsilon_i)_{1 \leq i \leq n}$ are i.i.d.~from $\mathcal{N}(0, \sigma^2)$. We test the distance-based strategy with different values of the threshold $thresh \in \{0.4, 0.6, 1\} $ and of the decrease factor $r \in \{1/2, 1/4, 1/8\} $. We use averaged-SGD with constant step size $\gamma=1/2R^2$  as a baseline since it enjoys the optimal statistical rate $O(\sigma^2 d / n)$~\citep{bach2013non}, we also plot SGD with step size $\gamma_n = 1 / \mu n$ which achieves a rate of $1 / \mu n$. 

We observe in \cref{app_fig:least_squares} that the distance-based strategy achieves similar performances as $1 / \mu n$ step sizes without knowing $\mu$. Furthermore the performance does not heavily depend on the values of $r$ and $thresh$ used. In the middle plot of \cref{app_fig:least_squares} notice how the distance-based step-sizes mimic the $1 / \mu n$ sequence. We point out that the performance of constant-step-size averaged SGD and $1/\mu n$-step-size  SGD are comparable since the problem is fairly well conditioned ($\mu=1/20$).

\paragraph{Logistic regression.} We consider the objective $f(\theta) = \E{\log ( 1 + e^{- y \ps{x}{\theta}})}$. The inputs $x_i$ are generated the same way as in the least-square setting. The outputs $y_i$ are generated following the logistic probabilistic model $y_i \sim \mathcal{B}((1 + \exp (- \ps{x_i}{\theta^*})^{-1})$. 
We use averaged-SGD with step-sizes $\gamma_n=1/ \sqrt{n}$  as a baseline since it enjoys the optimal rate $O(1/ n)$~\cite{bach2014adaptivity}. We also compare to online-Newton~\cite{bach2013non} which achieves better performance in practice and to averaged-SGD with step-sizes $\gamma_n=C/ \sqrt{n}$ where parameter $C$ is tuned in order to achieve best performance. 

In \cref{app_fig:logistic_regression} notice how averaged-SGD with the theoretical step size $\gamma_n=1/ \sqrt{n}$ performs poorly. However once the parameter $C$ in $\gamma_n=C/ \sqrt{n}$ is tuned properly averaged-SGD and online Newton perform similarly. Note that our distance-based strategy with $r = 1 / 2$ achieves similar performances which do not heavily depend on the value of the threshold $thresh$.

\paragraph{SVM.} We consider the objective $f(\theta) = \E {\max (0, 1 - y \ps{x}{\theta})} + \frac{\lambda}{2} \norm{\theta}^2$ where $\lambda > 0$. Note that $f$ is strongly-convex with parameter $\lambda$ and non-smooth. The inputs $x_i$ are generated i.i.d.~from $\mathcal{N}(0, \eta^2 I_d)$.  The outputs $y_i$ are generated as $y_i = \text{sgn} (x_i(1) + z_i)$ where $z_i \sim \mathcal{N}(0, \sigma^2)$. We generate $n = 1\text{e}5$ points in dimension $d = 20$.  We compare our distance-based strategy with different values of the threshold $thresh \in \{0.6, 0.8, 1\} $ to averaged-SGD with step sizes $\gamma_n = 1 / \mu n$ which achieves the rate of $\log n / \mu n$  \cite{lacoste2012simpler} and averaged-SGD with step sizes $\gamma_n = C / \sqrt{n}$ where $C$ is tuned in order to achieve best performance. 

In \cref{app_fig:SVM} note that averaged-SGD with $\gamma_n = 1 / \mu n$ exhibits a $O(1 / n)$ rate but the initial values are bad. On the other hand, once properly tuned, averaged SGD with $\gamma_n = C / \sqrt{n}$ performs very well, similarly as in the smooth setting. Note that our distance-based strategy with $r = 1 / 2$ achieves similar performances which do not depend on the value of the threshold $thresh$. 

\begin{figure}[t]
\begin{center}
\includegraphics[trim = {5cm 0cm 5cm 0cm}, clip, width=\linewidth]{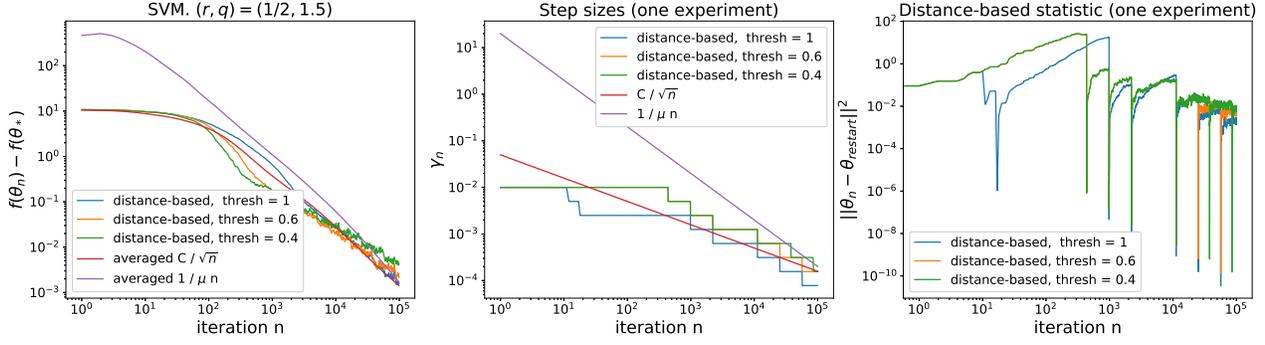}
\end{center}
\caption{SVM on synthetic data ($n = 1\text{e}5$, $d = 20$, $\lambda = 0.1$, $\eta^2 = 25$ and $\sigma = 1$). Distanced-based parameters: $(r, q, k_0) = (1 / 2, 1.5, 5)$ and $\gamma_0 = 4 / R^2$. The losses on the left plot are averaged over $10$ replications.}
\label{app_fig:SVM}
\end{figure}
\begin{figure}[t]
\begin{center}
\includegraphics[trim = {5cm 0cm 5cm 0cm}, clip, width=\linewidth]{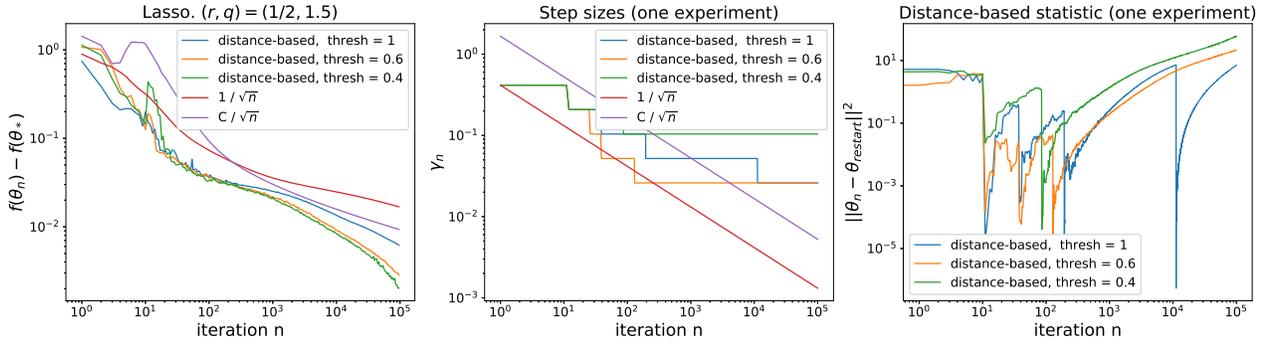}
\end{center}
\caption{Lasso regression on synthetic data ($\text{number of iterations} = 1\text{e}5$, $n = 80$, $d = 100$, $s = 60$, $\sigma = 0.1$, $\lambda = 10^{-4}$). Initial step-sizes of $1 / 2 R^2$ (except for the tuned $C / \sqrt{n}$).  Distanced-based parameters: $(r, q, k_0) = (1 / 2, 1.5, 5)$. The losses on the left plot are averaged over $10$ replications.}
\label{app_fig:lasso}
\end{figure}
\begin{figure}[t]
\begin{center}
\includegraphics[trim = {5cm 0cm 5cm 0cm}, clip, width=\linewidth]{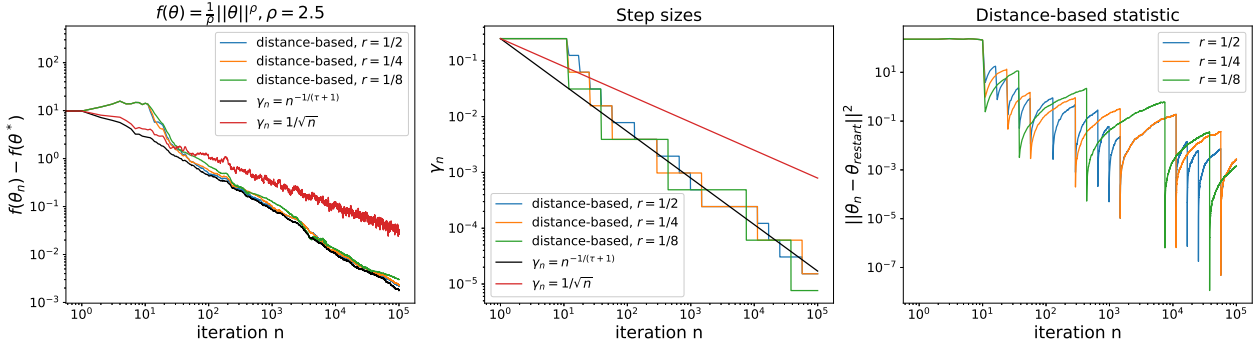}
\end{center}
\caption{Uniformly convex function $f(\theta) = \frac{1}{\rho} \norm{\theta}_2^{\rho}$ ($n = 1\text{e}5$, $d = 200$, $\rho = 2.5$). Initial step size of $\gamma_0 = 1 / 4 L$ for all step-size sequences. Distance-based parameters $(thresh, q, k_0) = (1, 1.5, 5)$.  The losses on the left plot correspond to only one replication. } 
\label{app_fig:uniformly_convex}
\end{figure}

\paragraph{Lasso Regression.} We consider the objective $f(\theta) = \frac{1}{n} \sum_{i = 1}^{n} (y_i - \ps{x_i}{\theta})^2 + \lambda \norm{\theta}_1$. The inputs $x_i$ are i.i.d.~from $\mathcal{N}(0, H)$ where $H$ has random eigenvectors and eigenvalues $(1 / k^3)_{1\leq k \leq d}$. We choose $n = 80$, $d = 100$. We note $R^2 = \tr{H}$. The outputs $y_i$ are generated following 
$y_i = \ps{x_i}{\tilde{\theta}} + \varepsilon_i$ 
where $(\varepsilon_i)_{1 \leq i \leq n}$ are i.i.d.~from $\mathcal{N}(0, \sigma^2)$ and $\tilde{\theta}$ is an $s$-sparse vector. Note that $f$ is non-smooth and the smallest eigenvalue of $H$ is $1 / 10^6$, hence for the number of iterations we run SGD $f$ cannot be considered as strongly convex.  We compare the distance-based strategy with different values of the threshold $thresh \in \{0.4, 0.6, 1\} $ to SGD with step-size sequence $\gamma_n = 1 / \sqrt{n}$ which achieves a rate of $\log n / \sqrt{n}$ \cite{shamir2013stochastic} and to  step-size sequence $\gamma_n = C / \sqrt{n}$ where $C$ is tuned to achieve best performance. Let us point out that the purpose of this experiment is to investigate the performance of the distance-based statistic on non-smooth problems and therefore we use as baseline generic algorithms for non-smooth optimization -- even though, in the special case of the Lasso regression, there exists first-order proximal algorithms which are able to leverage the special structure of the problem and obtain the same performance as for smooth optimization~\citep{BecTeb09}.

In \cref{app_fig:lasso} note that SGD with the theoretical step-size sequence $\gamma_n=1/ \sqrt{n}$ performs poorly. Tuning the parameter $C$ in $\gamma_n=C/ \sqrt{n}$ improves the performance. However our distance-based strategy with $r = 1 / 2$ performs better for several different values of $thresh$.

\paragraph{Uniformly convex $f$.} We consider the objective $f(\theta) = \frac{1}{\rho} \norm{\theta}_2^{\rho}$ where $\rho = 2.5$. Notice that $f$ is not strongly convex but is uniformly convex with parameter $\rho$ (see \cref{as:uc}). We generate the noise on the gradients $\xi_i$ as i.i.d from $\mathcal{N}(0, I_d)$.  We compare the distance-based strategy with different values of the decrease factor $r \in \{1/2, 1/4, 1/8\} $ to SGD with step-size sequence $\gamma_n = 1 / \sqrt{n}$ which achieves a rate of $\log n / \sqrt{n}$ \cite{shamir2013stochastic} and to SGD with step size $\gamma_n = n^{- 1 / (\tau + 1)}$ ($\tau = 1 - 2 / \rho$) which we expect to achieve a rate of $O(n^{- 1 / (\tau + 1)} \log n) $ (see remark after \cref{corr:appendix_uc_rate}). 
Notice in \cref{app_fig:uniformly_convex} how the distance-based strategy achieves the same rate as SGD with step-sizes $\gamma_n = n^{- 1 / (\tau + 1)}$ without knowing parameter $\tau$. Furthermore the performance does not depend on the value of $r$ used. In the middle plot of \cref{app_fig:least_squares} notice how the distance-based step sizes mimic the $n^{- 1 / (\tau + 1)}$ sequence.

Therefore the distance-based diagnostic works in a variety of settings where it automatically adapts to the problem difficulty without having to know the specific parameters (such as strong-convexity or uniform-convexity parameters).

\section{Performance of the oracle diagnostic}\label{secapp:oracle}

In this section, we prove the performance of the oracle diagnostic in the strongly-convex setting and consider its extension to the uniformly-convex setting.

\subsection{Proof of \cref{eq:oracle_rate}} \label{appsec:proofoforacle_rate}

We first introduce some notations which are useful in the following analysis.
\paragraph{Notation.} For $k \geq 1$, let $n_{k + 1}$ be the number of iterations until the $(k+1)^{th}$ restart and $\Delta n_{k+1}$ be the number of iterations between the restart $k$ and restart $(k+1)$ during which step size $\gamma_{k}$ is used. Therefore we have that $n_k = \sum_{k' = 1}^{k} \Delta n_{k'}$. We also denote by $\delta_n = \E{\norm{\theta_n - \theta^*}^2}$.

Notice that for  $n \geq 1$ and $| x | \leq n$ it holds that  $(1 - x)^n \leq \exp (- n x)$. Hence \cref{eq:needell} leads to:
\begin{align}
\label{eq:needell_ineq}
\E{\norm{\theta_{n} - \theta^*}^2} &\leq (1 - \gamma \mu )^{n } \delta_0 + \frac{2 \sigma^2}{\mu} \gamma \\
&\leq \exp (- n \gamma \mu ) \delta_0 + \frac{2 \sigma^2}{\mu} \gamma. \label{eq:needell_with_exponential}
\end{align}
In order to simplify the computations, we analyse \cref{alg:or} with the bias-variance trade-off stated in~\cref{eq:needell_with_exponential} instead of the one of~\cref{eq:needell_ineq}. Note however that it does not change the result. We prove separately the results obtained before and after the first restart $\Delta n_1$.

\paragraph{Before the first restart.} Let $\theta_0 \in \R^d$. For $n \leq \Delta n_1 = n_1$ (first restart time) we have that:
\begin{align}
\label{eq:beginning}
\E{\norm{\theta_{n} - \theta^*}^2} &\leq \exp (- n \gamma_0 \mu ) \delta_0 + \frac{2 \sigma^2}{\mu} \gamma_0.
\end{align}
Following the oracle strategy, the restart time $\Delta n_1$ corresponds to $\exp (- \Delta n_1 \gamma_0 \mu ) \delta_0 = \frac{2 \sigma^2}{\mu} \gamma_0$. Hence $\Delta n_1=\frac{1}{\gamma_0\mu} \ln \left ( \frac{\mu \delta_0}{2\gamma_0 \sigma^2} \right )$ and $\delta_{n_1} \leq \exp (- \Delta n_1 \gamma_0 \mu ) \delta_0 + \frac{2 \sigma^2}{\mu} \gamma_0 = \frac{4 \sigma^2}{\mu} \gamma_0$.
\paragraph{After the first restart.} Let $k \geq 1$ and $n_k \leq n \leq n_{k+1}$. We obtain from \cref{eq:needell_with_exponential}:
$$\begin{aligned}
\E{\norm{\theta_{n} - \theta^*}^2} &\leq \exp ( - (n - n_k) \gamma_k \mu ) \E{\norm{\theta_{n_k} - \theta^*}^2} + \frac{2 \sigma^2}{\mu} \gamma_k.
\end{aligned}$$
The oracle construction of the restart time leads to:
$$\exp ( - \Delta n_{k+1} \gamma_k \mu) \delta_{n_k} = \frac{2 \sigma^2}{\mu} \gamma_k.$$
Which yields
$$\begin{aligned}
\Delta n_{k+1} &= \frac{1}{ \gamma_k \mu } \ln{\frac{\mu \delta_{n_k}}{2 \sigma^2 \gamma_k}}. 
\end{aligned}$$

However we know by construction that for $k \geq 1$, $\delta_{n_k} \leq \exp ( - \Delta n_{k} \gamma_{k - 1} \mu ) \delta_{n_{k - 1}} + \frac{2 \sigma^2}{\mu} \gamma_{k -1} = \frac{4 \sigma^2}{\mu} \gamma_{k - 1}$. Hence: 
$$\begin{aligned}
\Delta n_{k+1} &\leq \frac{1}{\gamma_k \mu} \ln 2 \frac{\gamma_{k-1}}{\gamma_k}.
\end{aligned}$$

Considering that $\gamma_k = r^k \gamma_0$,
$$\begin{aligned}
\Delta n_{k+1} &\leq \frac{1}{r^k \gamma_0 \mu} \ln \frac{2}{r}.
\end{aligned}$$

Since $n_k = \Delta n_{1} + \sum_{k' = 2}^{k} \Delta n_{k'}$ we have that 

$$\begin{aligned}
n_k - \Delta n_{1} = \sum_{k' = 2}^{k} \Delta n_{k'} &\leq   \frac{1}{ \mu \gamma_0 } \ln \left ( \frac{2}{r} \right ) \sum_{k' = 2}^{k} \frac{1}{r^{k' - 1}} \\
&\leq   \frac{1}{ \mu \gamma_0 } \ln \left ( \frac{2}{r} \right ) \sum_{k' = 1}^{k} \frac{1}{r^{k' - 1}} \\
&\leq \frac{1}{\mu \gamma_0 (1 - r)} \ln \left ( \frac{2}{r} \right ) \frac{1}{r^{k-1}} \\
&= \frac{1}{\mu (1 - r) \gamma_{k - 1}} \ln \left ( \frac{2}{r} \right ).
\end{aligned}$$

Therefore since
$\delta_{n_k} \leq \frac{4 \sigma^2}{\mu} \gamma_{k - 1}$ we get:

\begin{equation}
\label{eq:ub_restart_times}
\delta_{n_k} \leq \frac{4\sigma^2}{(n_k - \Delta n_1) \mu^2 (1 - r)} \ln{\Big( \frac{2}{ r }\Big)}.
\end{equation}

We now want a result for any $n$ and not only for restart times. For $n \leq n_1 = \Delta n_1$ we are done using \cref{eq:beginning}. For $k \geq 1$, let $n_{k} \leq n \leq n_{k + 1}$, from \cref{eq:needell} and \cref{eq:ub_restart_times} we have that:
\begin{align*}
\delta_{n}
&\leq \exp ( - (n - n_k) \gamma_k \mu ) \delta_{n_{k}} + \frac{2 \gamma_k \sigma^2}{\mu} \\
&\leq \exp ( - (n - n_k) \gamma_k \mu ) \frac{A}{n_{k} - \Delta n_{1}} + \frac{2 \gamma_k \sigma^2}{\mu},
\end{align*}
where $A = \frac{4\sigma^2}{\mu^2 (1 - r)} \ln{\Big( \frac{2}{ r }\Big)}$. Let $g(n) = \exp ( - (n - n_k) \gamma_k \mu ) \frac{A}{n_{k} - \Delta n_{1}} + \frac{2 \gamma_k \sigma^2}{\mu}$ and $h(n) = \frac{A}{n - \Delta n_{1}} + \frac{2 \gamma_k \sigma^2}{\mu}$ for $n > \Delta n_1$. Note that $g$ is exponential, $h$ is an inverse function and that  $g(n_k) = h(n_k)$. This implies that that for $n \geq n_k$, $g(n) \leq h(n)$. Hence for $ n \geq n_{k}$:
\begin{align*}
\delta_{n}
&\leq \frac{A}{n - \Delta n_{1}} + \frac{2 \gamma_k \sigma^2}{\mu} \\
&\leq \frac{A}{n - \Delta n_{1}} + \frac{4 \gamma_k \sigma^2}{\mu}.
\end{align*}
By construction, $\frac{4 \sigma^2}{\mu} \gamma_{k} \leq \frac{A}{n_{k + 1} - \Delta n_{1}}$. However since 
$ \frac{A}{n_{k + 1} - \Delta n_{1}} \leq \frac{A}{n - \Delta n_{1}}$ for $n \leq n_{k+1}$ we get that $\frac{4 \sigma^2}{\mu} \gamma_{k} \leq \frac{A}{n - \Delta n_{1}}$ for $n \leq n_{k+1}$.
Hence for $ n_{k} \leq n \leq n_{k + 1}$ and therefore for all $n > \Delta n_1$:
\begin{align*}
\delta_{n}
&\leq \frac{2 A}{n - \Delta n_{1}} \\
&\leq \frac{8 \sigma^2}{\mu^2 (n - \Delta n_{1}) (1 - r)} \ln{\Big( \frac{2}{ r }\Big)}.
\end{align*}
This concludes the proof. Note that this upper bound diverges for $r \to 0 \ \text{or} \ 1$ and could be minimized over the value of $r$.

\subsection{Uniformly convex setting}\label{appsec:unif_convex}

The previous result holds for smooth strongly-convex functions. Here we extend this result to a more generic setting where $f$ is not supposed strongly convex but uniformly convex.

\begin{assumption}[Uniform convexity] 
\label{as:uc}
There exists finite constants  $\mu > 0, \rho > 2$ such that for all $\theta, \eta \in \mathbb R ^d$ and any subgradient $f'(\eta)$ of $f$ at $\eta$:
\[
f(\theta) \geq f(\eta) + \ps{f'(\eta)}{\theta - \eta} + \frac{\mu}{\rho} \norm{\theta - \eta}^\rho.
\]
\end{assumption}

This assumption implies the convexity of the function $f$ and the definition of strong convexity is recovered for $\rho \to 2$. It also recovers the definition of weak-convexity around $\theta^*$ when $\rho \to + \infty$ since $\textstyle \lim_{\rho \to + \infty}\frac{\mu}{\rho} \norm{\theta - \theta^*}^\rho =0$ for $\norm{\theta - \theta^*} \leq 1$.

To simplify our presentation and as is often done in the literature we restrict the analysis to the constrained  optimization problem:
\[
\min_{\theta \in \mathcal W} f(\theta),
\]
where $\mathcal W$ is a compact convex set and we assume $f$ attains its minimum on $\mathcal W$ at a certain  $\theta_*\in\mathbb{R}^d$. We consider the projected SGD recursion:
\begin{equation}{\label{eq:projected_sgd}}\theta_{i+1} = \Pi_{\mathcal{W}} \left [ \theta_i - \gamma_{i+1} f'_{i+1}(\theta_i) \right ].
\end{equation}
 We also make the following assumption (which does not contradict \cref{as:uc} in the constrained setting).
\begin{assumption}[Bounded gradients] 
\label{as:bounded_gradients}
There exists a finite constant $G > 0$ such that 
$$\E{\norm{f'_{i}(\theta)}^2} \leq G^2 $$
for all $i \geq 0$ and $\theta \in \mathcal{W}$. 
\end{assumption}

In order to obtain a result similar to \cref{eq:oracle_rate} but for uniformly convex functions, we first need to analyse the behaviour of constant step-size SGD in this new framework and obtain a classical bias-variance trade off similar to \cref{eq:needell}. 

\subsubsection{Constant step-size SGD for uniformly convex functions}

The following proposition exhibits the bias-variance trade off obtained for the function values when constant step-size SGD is used on uniformly convex functions.

\begin{proposition} 
\label{prop:uc_bias_variance_tradeoff}
 Consider the recursion in \cref{eq:projected_sgd} under \cref{as:unbiased_gradients,as:bounded_gradients,as:uc}. Let $\tau = 1 - \frac{2}{\rho} \in (0, 1)$,  $q = (\frac{1}{\tau} - 1)^{-1}$, $\tilde{\mu} = 4 \frac{\mu}{\rho}$ and $\delta_0= \E{\Vert \theta_0-\theta^*\Vert^2}$. Then for any step-size $\gamma > 0$ and time $n \geq 0$ we have:
\begin{equation} \nonumber
\E{f(\theta_n)} - f(\theta^*) 
\leq \frac{\delta_0}{ \gamma n \left ( 1 + n q \gamma \tilde{\mu} \delta_0^{q} \right )^{\frac{1}{q}} } + \gamma G^2 ( 1 + \log n) .
\end{equation}
\end{proposition}

Note that the bias term decreases at a rate $n^{- 1 / \tau}$ which is an interpolation of the rate obtained when $f$ is strongly convex ($\tau \to 0$, exponential decrease of bias) and when $f$ is simply convex ($\tau = 1$, bias decrease rate of $n^{-1}$). This bias-variance trade off directly implies the following rate in the finite horizon setting.

\begin{corollary}
\label{corr:appendix_uc_rate}
Consider the recursion in \cref{eq:projected_sgd} under \cref{as:unbiased_gradients,as:bounded_gradients,as:uc}. Then for a finite time horizon $N\geq0$ and constant step size $\gamma = N^{- \frac{1}{\tau + 1}}$  we have:
\begin{equation*}
\E{f(\theta_N)} - f(\theta^*)= O \left ( N^{-\frac{1}{1 + \tau}}  \log N  \right ) .
\end{equation*}
\end{corollary}

\paragraph{Remarks.} When the total number of iterations $N$ is fixed, \citet{juditsky2014deterministic} find a similar result as \cref{corr:appendix_uc_rate} for minimizing uniformly convex functions.
However their algorithm uses averaging and multiple restarts. In the deterministic framework, using a weaker but similar assumption as uniform convexity, \citet{roulet2017sharpness} obtain a similar $O(N^{- \frac{1}{\tau}})$ convergence rate for gradient descent for smooth uniformly convex functions. This is coherent with the bias variance trade off we get and \cref{corr:appendix_uc_rate} extends their result to the stochastic framework. We also note that the result in \cref{corr:appendix_uc_rate} holds only in the fixed horizon framework, however we believe that this rate still holds when using a decreasing step size $\gamma_n = n^{- \frac{1}{\tau + 1}}$. The analysis is however much harder since it requires analysing the recursion stated in \cref{reccursion:uc} with a decreasing step-size sequence.

Hence \cref{corr:appendix_uc_rate} shows that an accelerated rate of $O \left ( \log(n) n^{- \frac{1}{1 + \tau}}  \right )$ is obtained with appropriate step sizes. However in practice the parameter $\rho$ is unknown and this step size sequence cannot be implemented. In \cref{appsec:unif_convex_restarts} we show that we can bypass $\rho$ by using the oracle restart strategy. In the following subsection \cref{appsec:unif_convex_proofs} we prove \cref{prop:uc_bias_variance_tradeoff} and \cref{corr:appendix_uc_rate}.

\subsubsection{Proof of \cref{prop:uc_bias_variance_tradeoff} and \cref{corr:appendix_uc_rate}} \label{appsec:unif_convex_proofs}

We start by stating the following lemma directly inspired by \citet{shamir2013stochastic}.

\begin{lemma}
\label{lemma:shamir_zhang}
Under \cref{as:unbiased_gradients,as:bounded_gradients}. Consider projected SGD in \cref{eq:projected_sgd} with constant step size $\gamma >0$. Let $1 \leq p \leq n$ and denote $S_p = \frac{1}{p+1} \sum_{i = n-p}^n f(\theta_i)$, then:
$$\E{f(\theta_n)} \leq \E{S_p} + \frac{\gamma}{2} G^2 (\log(p) + 1).$$
\end{lemma}

\begin{proof}
We follow the proof technique of \citet{shamir2013stochastic}. The goal is to link the value of the final iterate with the averaged last $p$ iterates. For any $\theta \in \mathcal{W}$ and $\gamma > 0$:

$$\theta_{i+1} - \theta = \Pi_{\mathcal{W}} \left [ \theta_i - \gamma f'_{i+1}(\theta_i) \right ]- \theta.$$
By convexity of $\mathcal{W}$ we have the following: 
\begin{align}
\norm{\theta_{i+1} - \theta}^2 
&\leq
\norm{ \theta_i - \gamma f'_{i+1}(\theta_i) - \theta  }^2 \nonumber \\
&=
\norm{\theta_{i} - \theta}^2 
-
2 \gamma \ps{f'_{i+1}(\theta_i)}{\theta_i - \theta} 
+ 
\gamma^2 \norm{f'_{i+1}(\theta_i)}^2.  \label{eq:1} 
\end{align}
Rearranging we get
\begin{align}  \label{eq:2} 
\ps{f'_{i+1}(\theta_i)}{\theta_i - \theta} 
\leq 
\frac{1}{2 \gamma} \left [ \norm{\theta_{i} - \theta}^2  
-  
\norm{\theta_{i+1} - \theta}^2 \right ] 
+
\frac{\gamma}{2} \norm{f'_{i+1}(\theta_i)}^2
.
\end{align}
Let $k$ be an integer smaller than $n$. Summing \cref{eq:2} from $i = n - k$ to $i = n$ we get
$$\sum_{i = n-k}^n \ps{f'_{i+1}(\theta_i)}{\theta_i - \theta} 
\leq 
\frac{1}{2 \gamma} \left [ \norm{\theta_{n-k} - \theta}^2  
-  
\norm{\theta_{n+1} - \theta}^2 \right ] 
+
\frac{\gamma}{2} \sum_{i = n-k}^n  \norm{f'_{i+1}(\theta_i)}^2 .
$$
Taking the expectation and using the bounded gradients hypothesis:
$$\begin{aligned}
\sum_{i = n-k}^n \E{\ps{f'(\theta_i)}{\theta_i - \theta}}
&\leq
\frac{1}{2 \gamma} \E{\norm{\theta_{n-k} - \theta}^2
-  
\norm{\theta_{n+1} - \theta}^2} 
+
\frac{\gamma}{2} \sum_{i = n-k}^n  \E{\norm{f'_{i+1}(\theta_i)}^2} \\
& \leq 
\frac{1}{2 \gamma} \E{\norm{\theta_{n-k} - \theta}^2
-  
\norm{\theta_{n+1} - \theta}^2} 
+
\frac{\gamma}{2} (k+1) G^2 .
\end{aligned}
$$
The function $f$ being convex we have that $f(\theta_i) - f(\theta) \leq \ps{f'(\theta_i)}{\theta_i - \theta}$.
Therefore:
$$\begin{aligned}
\frac{1}{k+1} \sum_{i = n-k}^n \E{f(\theta_i) - f(\theta)}
&\leq 
\frac{1}{2 \gamma (k + 1)} \E{\norm{\theta_{n-k} - \theta}^2
-  
\norm{\theta_{n+1} - \theta}^2} 
+
\frac{\gamma}{2} G^2 \\
&\leq 
\frac{1}{2 \gamma (k + 1)} \E{\norm{\theta_{n-k} - \theta}^2}
+
\frac{\gamma}{2} G^2 .
\end{aligned}
$$
Let $S_k = \frac{1}{k+1} \sum_{i = n-k}^n f(\theta_i)$. Rearranging the previous inequality we get
\begin{align}  \label{eq:3} 
\E{S_k} - f(\theta) 
&\leq   
\frac{1}{2 \gamma (k + 1)} \E{\norm{\theta_{n-k} - \theta}^2}
+
\frac{\gamma}{2} G^2  \nonumber \\
&\leq   
\frac{1}{2 \gamma k } \E{\norm{\theta_{n-k} - \theta}^2}
+
\frac{\gamma}{2} G^2 .
\end{align}
Plugging $\theta = \theta_{n-k}$ in \cref{eq:3} we get
$$\begin{aligned}
- \E{f(\theta_{n-k})}
&\leq 
- \E{S_k}
+
\frac{\gamma}{2} G^2 .
\end{aligned}
$$
However, notice that
$k \E{S_{k-1}} = (k+1) \E{S_k} - \E{f(\theta_{n-k})}$.
Therefore:
$$\begin{aligned}
k \E{S_{k-1}} &\leq (k+1) \E{S_k}
- \E{S_k}
+
\frac{\gamma}{2} G^2 \\
&= k \E{S_k} + 
\frac{\gamma}{2} G^2 .
\end{aligned}$$
Summing the inequality $\E{S_{k-1}} \leq \E{S_k} + 
\frac{\gamma}{2 k} G^2$ from $k=1$ to some $p \leq n$ we get $\E{S_0} \leq \E{S_p} + \frac{\gamma}{2} G^2 \sum_{k=1}^{p} \frac{1}{k}$. Since $S_0 = f(\theta_n)$ we have the 
following inequality that links the final iterate and the averaged last $p$ iterates:
\begin{align}\label{eq:4}
\E{f(\theta_n)} \leq \E{S_p} + \frac{\gamma}{2} G^2 (\log(p) + 1).
\end{align}
\end{proof}

The inequality~\eqref{eq:4} shows that upper bounding $\E{S_p}$ immediately gives us an upper bound on $\E{f(\theta_n)}$. This is useful because it is often simpler to upper bound the average of the function values $\E{S_p}$ than directly $\E{f(\theta_n)}$. Therefore to prove \cref{prop:uc_bias_variance_tradeoff} we now just have to suitably upper bound $\E{S_p}$.

\begin{proof}[Proof of \cref{prop:uc_bias_variance_tradeoff}]
The function $f$ is uniformly convex with parameters $\mu > 0$ and $\rho> 2$ which means that for all $\theta, \eta \in \mathcal{W}$ and any subgradient $f'(\eta)$ of $f$ at $\eta$ it holds that
$f(\theta) \geq f(\eta) + \ps{f'(\eta)}{\theta - \eta} + \frac{\mu}{\rho} \norm{\theta - \eta}^{\rho}$.   
Adding this inequality written in $(\theta, \eta)$ and in $(\eta, \theta)$ we get:
\begin{align}
\label{eq:6}
\frac{2 \mu}{\rho} \norm{\theta - \eta}^{\rho} \leq \ps{f'(\theta) - f'(\eta)}{\theta - \eta} .
\end{align}
Using inequality \sref{eq:1} with $\theta = \theta^*$ and taking its expectation we get that 
$$\delta_{n+1}
\leq
\delta_{n}
-
2 \gamma \E{\ps{f'(\theta_n)}{\theta_n - \theta} }
+ 
\gamma^2 G^2 .
$$
Therefore using inequality from \cref{eq:6} with $\eta = \theta^*$:
\begin{align}
\delta_{n+1} &\leq \delta_{n}
-
4 \gamma \frac{\mu}{\rho} \E{ \norm{\theta_n - \theta^*}^{\rho}} 
+ 
\gamma^2 G^2 . \nonumber
\end{align}
Since $\rho > 2$ we use Jensen's inequality to get $\E{ \norm{\theta_n - \theta^*}^{\rho}} \geq \E{ \norm{\theta_n - \theta^*}^2}^{\rho / 2}$. Let $\tilde{\mu} = 4 \frac{\mu}{\rho}$, then:
\begin{align}
\label{reccursion:uc}
\delta_{n+1}
&\leq \delta_{n}
-
4 \gamma \frac{\mu}{\rho} \delta_n^{\frac{\rho}{2}}
+ 
\gamma^2 G^2 \nonumber \\
&= \delta_{n}
-
\gamma \tilde{\mu} \delta_n^{\frac{\rho}{2}}
+ 
\gamma^2 G^2 .
\end{align}
Let $\tilde{g} : x \in \R_+ \mapsto x - \gamma \tilde{\mu} x^{\rho / 2}$. The function $\tilde{g}$ is strictly increasing on $[0, x_c]$ where $x_c = \left ( \frac{2}{\rho \gamma \tilde{\mu}} \right )^{2 / (\rho - 2)}$. 
Let $\delta_{\infty} = (\frac{\gamma G^2}{\tilde{\mu}})^{\frac{2}{\rho}}$ such that $\tilde{g}(\delta_{\infty}) + \gamma^2 G^2 = \delta_{\infty}$. 
We assume that $\gamma$ is small enough so that $\delta_{\infty} < x_c$.
Therefore if $\delta_0 \leq x_c$ then $\delta_n \leq x_c$ for all $n$. By recursion we now show that:
\begin{align}
\label{eq:8}
\delta_{n} \leq \tilde{g}^{n}(\delta_0) + n \gamma^2 G^2 .
\end{align}
Inequality \sref{eq:8} is true for $n=0$. Now assume inequality \sref{eq:8} is true for some $n \geq 0$. 
According to \cref{reccursion:uc}, $\delta_{n+1} \leq \tilde{g}(\delta_n) + \gamma^2 G^2$. If $\tilde{g}^{n}(\delta_0) + n \gamma^2 G^2 > x_c$ then we immediately get 
$\delta_{n+1} \leq x_c < \tilde{g}^{n}(\delta_0) + (n+1) \gamma^2 G^2$ and recurrence is over. Otherwise, since $\tilde{g}$ is increasing on $[0, \ x_c]$ we have that 
$\tilde{g}(\delta_n) \leq \tilde{g}(\tilde{g}^{n}(\delta_0) + n \gamma^2 G^2 )$ and:
$$\begin{aligned}
\delta_{n+1} &\leq \tilde{g}(\tilde{g}^{n}(\delta_0) + n \gamma^2 G^2 ) +  \gamma^2 G^2 \\ 
&= \left [ \tilde{g}^{n}(\delta_0) + n \gamma^2 G^2 \right ] - \gamma \tilde{\mu} \left [ \tilde{g}^{n}(\delta_0) + n \gamma^2 G^2 \right ]^{\rho / 2} + \gamma^2 G^2 \\
&\leq \tilde{g}^{n}(\delta_0) - \gamma \tilde{\mu} \left [ \tilde{g}^{n}(\delta_0)  \right ]^{\rho / 2} + (n+1) \gamma^2 G^2 \\
&= \tilde{g}^{n + 1}(\delta_0) + (n+1) \gamma^2 G^2  .
\end{aligned}$$
Hence \cref{eq:8} is true for all $n \geq 0$.
Now we analyse the sequence $\left ( \tilde{g}^{n}(\delta_0) \right )_{n \geq 0}$. Let $\tilde{\delta}_n = \tilde{g}^{n}(\delta_0)$. Then
 $0 \leq \tilde{\delta}_{n+1} = \tilde{\delta}_{n} - \gamma \tilde{\mu} \tilde{\delta}_{n}^{q + 1} \leq \tilde{\delta}_{n}$ where $q = \rho / 2 - 1 > 0$. 
Therefore $\tilde{\delta}_{n}$ is decreasing, lower bounded by zero, hence it convergences to a limit which in our case can only be $0$. Note that $(1 - x)^{- q} \geq 1 + q x$ for $q > 0$ and $x < 1$. Therefore:
$$\begin{aligned}
\left ( \tilde{\delta}_{n+1} \right )^{-q} &= (\tilde{\delta}_{n} - \gamma \tilde{\mu} \tilde{\delta}_{n}^{q + 1})^{-q} \\
&= \tilde{\delta}_{n}^{-q} (1 - \gamma \tilde{\mu} \tilde{\delta}_{n}^{q})^{-q} \\
&\geq \tilde{\delta}_{n}^{-q} (1 + q \gamma \tilde{\mu} \tilde{\delta}_{n}^{q}) \\
&= \tilde{\delta}_{n}^{-q}  + q \gamma \tilde{\mu} .
\end{aligned}$$
Summing this last inequality we obtain:
$\tilde{\delta}_{n}^{-q} \geq \tilde{\delta}_{0}^{-q}  + n q \gamma \tilde{\mu}$ which leads to
$$\begin{aligned}
\tilde{\delta}_{n} &\leq (\tilde{\delta}_{0}^{-q}  + n q \gamma \tilde{\mu})^{-1 / q} \\
&= (\delta_{0}^{-q}  + n q \gamma \tilde{\mu})^{-1 / q} .
\end{aligned}$$
Therefore:
$$\begin{aligned}
\delta_n &\leq \delta_0 \left ( 1 + n q \gamma \tilde{\mu} \delta_0^{q} \right )^{- \frac{1}{q}}  + n \gamma^2 G^2 \\
&= \frac{\delta_0}{ \left ( 1 + n q \gamma \tilde{\mu} \delta_0^{q} \right )^{\frac{1}{q}} } +  n \gamma^2 G^2 \\
&\leq O \left ( \frac{1}{\gamma n}^{\frac{2}{\rho - 2}} \right ) + n \gamma^2 G^2 .
\end{aligned}$$
Plugging this in \cref{eq:3} with $k = n / 2$ and $\theta = \theta^*$ we get:
$$\begin{aligned}
\E{S_{n / 2}} -  f(\theta^*) &\leq \frac{1}{\gamma n} \delta_{n / 2} + \frac{\gamma}{2} G^2 \\
&\leq \frac{1}{\gamma n} \left (  \delta_0 \left ( 1 + \frac{n}{2} q \gamma \tilde{\mu} \delta_0^{q} \right )^{- \frac{1}{q}} + \frac{n}{2} \gamma^2 G^2   \right ) + \frac{\gamma}{2} G^2 \\
&= \frac{\delta_0}{ \gamma n \left ( 1 + \frac{1}{2} n q \gamma \tilde{\mu} \delta_0^{q} \right )^{\frac{1}{q}} } + \gamma G^2 \\
&\leq  O \left ( \frac{1}{(\gamma n)^{\frac{1}{\tau}}} \right )  + \gamma G^2 ,
\end{aligned}$$
where $\tau = 1 - \frac{2}{\rho} \in [0, 1]$.
Re-injecting this inequality in \cref{eq:4} with $p = n / 2$ we get:
$$\begin{aligned}
\E{f(\theta_n)} - f(\theta^*) 
&\leq  \frac{\delta_0}{ \gamma n \left ( 1 + n q \gamma \tilde{\mu} \delta_0^{q} \right )^{\frac{1}{q}} } + \gamma G^2 +  \frac{\gamma G^2}{2} \left ( \log(\frac{n}{2}) + 1 \right ) \\
&\leq \frac{\delta_0}{ \gamma n \left ( 1 + n q \gamma \tilde{\mu} \delta_0^{q} \right )^{\frac{1}{q}} } + \gamma G^2 +  \gamma G^2 \log(n) \quad \text{for } n \geq 2\\
&\leq  O \left ( \frac{1}{(\gamma n)^{\frac{1}{\tau}}} \right ) + \gamma G^2 (1 + \log(n)) .
\end{aligned}$$ 
\end{proof}

The proof of \cref{corr:appendix_uc_rate} follows easily from \cref{prop:uc_bias_variance_tradeoff}.
\begin{proof}[Proof of \cref{corr:appendix_uc_rate}]
In the finite horizon framework, by choosing $\gamma = \frac{1}{N^{\frac{1}{\tau + 1}}}$ we get that: 
\begin{align*}
\E{f(\theta_N)} - f(\theta^*) &\leq O \left ( \frac{1}{N^{\frac{1}{\tau + 1}}} \right )  +  G^2 \frac{1 + \log(N)}{N^{\frac{1}{\tau + 1}}} \nonumber \\
&= O \left ( \frac{\log(N)}{N^{\frac{1}{1 + \tau}}}  \right ) .
\end{align*}
\end{proof}

\subsubsection{Oracle restart strategy for uniformly convex functions.}\label{appsec:unif_convex_restarts}

As seen at the end of \cref{appsec:unif_convex}, appropriate step sizes can lead to accelerated convergence rates for uniformly convex functions. However in practice these step sizes are not implementable since $\rho$ is unknown. Here we study the oracle restart strategy which consists in decreasing the step size when the iterates make no more progress. To do so we consider the following bias trade off inequality which is verified for uniformly convex functions (\cref{prop:uc_bias_variance_tradeoff}) and for convex functions (when $\tau = 1$).

\begin{assumption}
\label{as:uc_bias_tradeoff}
There is a bias variance trade off on the function values for some $\tau \in (0, 1]$ of the type:
$$\E{f(\theta_n)} - f(\theta^*) \leq A \left ( \frac{1}{\gamma n} \right )^{\frac{1}{\tau}} + B \gamma (1 + \log(n)).$$
\end{assumption}
Under \cref{as:uc_bias_tradeoff}, if we assume the constants of the problem $A$ and $B$ are known then we can adapt \cref{alg:or} in the uniformly convex case. From $\theta_0 \in \mathcal{W}$ we run the SGD procedure with a constant step size $\gamma_0$ for $\Delta n_1$ steps until the bias term is dominated by the variance term. This corresponds to $A \left ( \frac{1}{\gamma_0 \Delta n_1} \right )^{\frac{1}{\tau}} =  B \gamma_0$. Then for $n \geq \Delta n_1$, we decide to use a smaller step size $\gamma_1 = r \times \gamma_0$ (where $r$ is some parameter in $[0, 1]$) and run the SGD procedure for $\Delta n_2$ steps until $A \left ( \frac{1}{\gamma_1 \Delta n_2} \right )^{\frac{1}{\tau}} = B \gamma_1$  and we reiterate the procedure. This mimics dropping the step size each time 
the final iterate has reached function value saturation. This procedure is formalized in \cref{alg:or_uc}.

\begin{algorithm}[h!]
\caption{Oracle diagnostic for uniformly convex functions}
\label{alg:or_uc}
\begin{algorithmic}
\STATE \textbf{Input: }  $\gamma$, $A$, $B$, $\tau$
\STATE \textbf{Output:} Diagnostic boolean
\STATE Bias $\gets A \left ( \frac{1}{\gamma n} \right )^{\frac{1}{\tau}}$ 
\STATE Variance $\gets B \gamma$
\STATE \textbf{Return: } \{ Bias $<$ Variance \}
\end{algorithmic}
\end{algorithm}

In the following proposition we analyse the performance of the oracle restart strategy for uniformly convex functions. The result is similar to \cref{eq:oracle_rate}.

\begin{proposition}
\label{eq:oracle_uc_rate}
Under \cref{as:uc_bias_tradeoff}, consider \Cref{alg:gen} instantiated  with \Cref{alg:or_uc} and parameter   $r \in (0, 1)$ . Let $\gamma_0 > 0$,
then for all restart times $n_k$: 
\begin{align}
\label{eq:uc_oracle_rate}
\E{f(\theta_{n_k})} - f(\theta^*) \leq O \left ( \log(n_k) n_k^{- \frac{1}{\tau + 1}}   \right ).
\end{align} 
\end{proposition}
Hence by using the oracle restart strategy we recover the rate obtained by using the step size $\gamma = n^{- \frac{1}{\tau + 1}}$. This suggests that efficiently detecting stationarity can result in a convergence rate that adapts to parameter $\rho$ which is unknown in practice, this is illustrated in \cref{app_fig:uniformly_convex}.
However, note that unlike the strongly convex case, \cref{eq:uc_oracle_rate} is valid only at restart times $n_k$. Our proof here resembles to the classical doubling trick. However in practice (see \cref{app_fig:uniformly_convex}), the rate obtained is valid for all $n$.

\begin{proof}
As before, for $k \geq 0$, denote by $n_{k+1}$ the number of iterations until the $(k + 1)^{th}$ restart and $\Delta n_{k + 1}$ the number of iterations between restart $k$ 
and restart $(k + 1)$ during which step size $\gamma_{k}$ is used. Therefore we have that $n_k = \sum_{k' = 1}^{k} \Delta n_{k'}$ and $\gamma_k = r^k \gamma_0$.

Following the restart strategy :
$$A \left ( \frac{1}{\gamma_k \Delta n_{k+1}} \right )^{\frac{1}{\tau}} =  B \gamma_k.$$
Rearranging this equality we get:
$$
\Delta n_{k+1} =  \frac{A}{ B}  \frac{1}{\gamma_k^{\tau + 1}} 
=  \frac{A}{ B } \frac{1}{\gamma_0^{\tau + 1}} \frac{1}{ r^{k (\tau + 1)}}.
$$
And,
$$\begin{aligned}
n_k = \sum_{k' = 1}^{k} \Delta n_{k'} &
=  \frac{A}{ B } \frac{1}{\gamma_0^{\tau + 1}}  \sum_{k' = 0}^{k-1} \frac{1}{r^{k' (\tau + 1)}} \\
&\leq \frac{A}{ B} \frac{1}{\gamma_0^{\tau + 1}}  \frac{r^{\tau + 1}}{1 - r^{\tau + 1}} \frac{1}{r^{k (\tau + 1)}} \\
&\leq \frac{A}{ B} \frac{1}{\gamma_0^{\tau + 1}}  \frac{1}{1 - r^{\tau + 1}} \frac{1}{r^{(k - 1) (\tau + 1)}} \\
&= \frac{A}{ B} \frac{1}{(\gamma_{k - 1})^{\tau + 1}}  \frac{1}{1 - r^{\tau + 1}} .
\end{aligned}$$
Since $\E{f(\theta_{n_k})} - f(\theta^*) \leq B \gamma_{k-1} ( 1 + \log(\Delta n_k))$ we get:
\begin{align*}
\E{f(\theta_{n_k})} - f(\theta^*) &\leq B \gamma_0 r^{k-1} ( 1 + \log(n_k)) \nonumber \\
&\leq B \left (\frac{A}{ B} \frac{1}{1 - r^{\tau + 1}} \right )^{\frac{1}{\tau + 1}} \frac{1}{n_k^{\frac{1}{\tau + 1}} } ( 1 + \log(n_k)) \nonumber  \\
&\leq O \left ( \frac{\log(n_k) }{n_k^{\frac{1}{\tau + 1}} }  \right ).
\end{align*} 
\end{proof}
\section{Analysis of Pflug's statistic}\label{appsec:pflugs}

In this section we prove \Cref{eq:inner_product_expansion} which shows that at stationarity the inner product $\ps{ f'_1(\theta_0)}{ f'_2(\theta_1)}$ is negative. We then prove \Cref{eq:pflug_probality} which shows that using Pflug's statistic leads to abusive and undesired restarts.

\subsection{Proof of \Cref{eq:inner_product_expansion}}

Let $f$ be an objective function verifying  \cref{as:unbiased_gradients,as:sc,as:smoothness,as:noise_condition,as:5_times_differentiability,as:supp_noise_conditions}. We first state the following lemma from \citet{dieuleveut2017bridging}. 

\begin{lemma}
\label{lemma:norm_eta_dieuleveut}
\textbf{[Lemma 13 of \citet{dieuleveut2017bridging}]}
Under \cref{as:unbiased_gradients,as:sc,as:smoothness,as:noise_condition,as:5_times_differentiability,as:supp_noise_conditions},
for $\gamma \leq 1 / 2L$: 
$$\Epi{\norm{\eta}^{2 p}} = O(\gamma^p).$$
Therefore by the Cauchy-Schwartz inequality: 
$\Epi{\norm{\eta}} \leq \Epi{\norm{\eta}^2}^{1/2}  = O(\sqrt{\gamma}).$
\end{lemma}
In the following proofs we use the Taylor expansions with integral rest of $f'$ around $\theta^*$ we also state here:
\paragraph{Taylor expansions of $f'$.} Let us define $\mathcal{R}_1$ and $\mathcal{R}_2$  such that for all $\theta \in \R^d$:
\begin{itemize}
    \item $f'(\theta) = f''(\theta^*)(\theta - \theta^*) + \mathcal{R}_1(\theta)$ where $\mathcal{R}_1 : \R^d \to \R^d$ satisfies $\underset{\theta \in \R^d}{\sup} \big( \frac{\norm{\mathcal{R}_{1}(\theta)}} {\norm{\theta - \theta^*}^2}\big) = M_1 < + \infty$
    \item $f'(\theta) = f''(\theta^*)(\theta - \theta^*) + f^{(3)}(\theta^*) (\theta - \theta^*)^{\otimes 2} + \mathcal{R}_2(\theta)$ where $\mathcal{R}_2 : \R^d \to \R^d$ satisfies $\underset{\theta \in \R^d}{\sup} \big(\frac{ \norm{\mathcal{R}_2(\theta)}} { \norm{\theta - \theta^*}^3}\big) = M_2 < + \infty$
\end{itemize}

We also make use of this simple lemma which easily follows from \cref{lemma:norm_eta_dieuleveut}.
\begin{lemma}{\label{lemma:expectation_f_prime}}
Under \cref{as:unbiased_gradients,as:sc,as:smoothness,as:noise_condition,as:5_times_differentiability,as:supp_noise_conditions}, let $\gamma \leq 1 / 2L$, then 
$\Epi{\norm{f'(\theta)}} = O(\sqrt{\gamma}).$
\end{lemma}

\begin{proof}
$f'(\theta) = f''(\theta^*) \eta + \mathcal{R}_1(\theta)$ so that $\Epi{\norm{f'(\theta)}} \leq \norm{f''(\theta^*)}_{\text{op}} \Epi{\norm{\eta}} + M_1 \Epi{\norm{\eta}^2}$. With \Cref{lemma:norm_eta_dieuleveut} we then get that $\Epi{\norm{f'(\theta)}} = O(\sqrt{\gamma})$.
\end{proof}

We are now ready to prove \cref{eq:inner_product_expansion}.

\paragraph{Proof of \cref{eq:inner_product_expansion}.}
For $\theta_0 \in \R^d$ we have that $f'_1(\theta_0) =  f'(\theta_0) - \varepsilon_1(\theta_0)$,  $\theta_1 = \theta_0 - \gamma f'_1(\theta_0)$ and $ f'_2(\theta_1) =  f'(\theta_1) - \varepsilon_2(\theta_1)$. Hence:

$$\ps{f'_1(\theta_0)}{f'_2(\theta_1)} = \ps{ f'_1(\theta_0)}{ f'(\theta_1) - \varepsilon_2(\theta_1)}.$$
And by \cref{as:unbiased_gradients},
\begin{align}
\label{eq:two_terms}
\E{\ps{ f'_1(\theta_0)}{ f'_2(\theta_1)} \ | \ \mathcal{F}_1} &= \ps{ f'_1(\theta_0)}{ f'(\theta_1)} \nonumber \\
&= \ps{ f'(\theta_0) - \varepsilon_1(\theta_0)}{ f'(\theta_0 - \gamma f'(\theta_0) + \gamma \varepsilon_1(\theta_0))} \nonumber \\
&= \underbrace{\ps{ f'(\theta_0)}{ f'(\theta_0 - \gamma f'(\theta_0) + \gamma \varepsilon_1(\theta_0))}}_{\text{"deterministic"}} - \underbrace{\ps{\varepsilon_1(\theta_0)}{ f'(\theta_0 - \gamma f'(\theta_0) + \gamma \varepsilon_1(\theta_0))}}_{\text{noise}}. 
\end{align}

\paragraph{First part of the proposition.}
By a Taylor expansion in $\gamma$ around $\theta_0$:
$$f'(\theta_0 - \gamma f'(\theta_0) + \gamma \varepsilon_1(\theta_0)) = f'(\theta_0) - \gamma f''(\theta_0) \left ( f'(\theta_0) - \varepsilon_1(\theta_0) \right ) + O(\gamma^2).$$
Hence: 
\begin{align*}
\E{\ps{ f'_1(\theta_0)}{ f'_2(\theta_1)}} &= \norm{f'(\theta_0)}^2 - \gamma \ps{f'(\theta_0)}{f''(\theta_0) f'(\theta_0)} - \gamma \E{\ps{\varepsilon_1(\theta_0)}{f''(\theta_0) \varepsilon_1(\theta_0)}} + O(\gamma^2) \\
&\geq (1 - \gamma L ) \norm{f'(\theta_0)}^2 - \gamma L \tr{\mathcal{C}(\theta_0)} + O(\gamma^2).
\end{align*}

\paragraph{Second part of the proposition.}
For the second part of the proposition we  make use of the Taylor expansions around $\theta^*$. \Cref{eq:two_terms} is the sum of two terms, a "deterministic" (note that we use brackets since the term is not exactly deterministic)  and a noise term, which we compute separately below. Let $\eta_0 = \theta_0 - \theta^*$. 

\paragraph{"Deterministic" term.} First,

$$\begin{aligned}
\ps{ f'(\theta_0)}{ f'(\theta_0 - \gamma f'(\theta_0) + \gamma \varepsilon_1(\theta_0))} &= \ps{f'(\theta_0)}{f''(\theta^*) \eta_0} \\
&- \gamma \ps{f'(\theta_0)}{f''(\theta^*) f'(\theta_0)} \\
&+ \gamma \ps{f'(\theta_0)}{f''(\theta^*) \varepsilon_1(\theta_0))} \\
& + \gamma \ps{f'(\theta_0)}{ \mathcal{R}_1(\theta_0  - \gamma f'(\theta_0) + \gamma \varepsilon_1(\theta_0))}.
\end{aligned}$$

We compute each of the four terms separately, for $\theta_0 \sim \pigamma$:
\paragraph{a)} 
$$\begin{aligned} 
\Epi{\ps{f'(\theta_0)}{f''(\theta^*) \eta_0}} &= \Epi{\ps{f''(\theta^*) \eta_0}{f''(\theta^*) \eta_0}} + \Epi{\ps{\mathcal{R}_1(\theta_0)}{f''(\theta^*) \eta_0}} \\
 &= \Epi{\eta_0^T f''(\theta^*)^2 \eta_0} + O(\gamma^{3/2}).
\end{aligned}$$
However 
\begin{align*} 
\Epi{\ps{\mathcal{R}_1(\theta_0)}{f''(\theta^*) \eta_0}} &\leq \Epi{ \norm{\mathcal{R}_1(\theta_0)} \norm{f''(\theta^*) \eta_0}} \\
 &\leq M_1 \Epi{ \norm{\eta_0}^2 \norm{f''(\theta^*) \eta_0}} \\
 &= O(\gamma^{3/2}) \qquad \text{by  \cref{lemma:norm_eta_dieuleveut}}.
\end{align*}
Hence $\Epi{\ps{f'(\theta_0)}{f''(\theta^*) \eta_0}} = \Epi{\eta_0^T f''(\theta^*)^2 \eta_0} + O(\gamma^{3/2})$.

\paragraph{b)} Using  \Cref{lemma:expectation_f_prime}:
$$ \gamma \Epi{\ps{f'(\theta_0)}{f''(\theta^*) f'(\theta_0)}} = O(\gamma^2). $$ 

\paragraph{c)} Using \cref{as:unbiased_gradients}: 
$$\Epi{\ps{f'(\theta_0)}{f''(\theta^*) \varepsilon_1(\theta_0)}} = 0.$$

\paragraph{d)} Using the Cauchy-Schwartz inequality, \cref{lemma:norm_eta_dieuleveut,lemma:expectation_f_prime} :

$$\Epi{|\ps{f'(\theta_0)}{ \mathcal{R}_1(\theta  - \gamma f'(\theta_0) + \gamma \varepsilon_1(\theta_0))}|} \leq M_1 \Epi{ \norm{f'(\theta_0)} \norm{\eta_0  - \gamma f'(\theta_0) + \gamma \varepsilon_1(\theta_0)}^2} = O(\gamma^{3/2}).$$

\paragraph{Noise term.}
Now we deal with the noise term in \cref{eq:two_terms}:

$$\begin{aligned}
\ps{\varepsilon_1(\theta_0)}{ f'(\theta_0 - \gamma f'(\theta_0) + \gamma \varepsilon_1(\theta_0))} &= \ps{\varepsilon_1(\theta_0)}{f''(\theta^*) (\eta_0 - \gamma f'(\theta_0) + \gamma \varepsilon_1(\theta_0)} \\
&+ \ps{\varepsilon_1(\theta_0)}{f^{(3)}(\theta^*) (\eta_0 - \gamma f'(\theta_0) + \gamma \varepsilon_1(\theta_0))^{\otimes 2}} \\
& + \ps{\varepsilon_1(\theta_0)}{ \mathcal{R}_2(\theta  - \gamma f'(\theta_0) + \gamma \varepsilon_1(\theta_0))}.
\end{aligned}$$

We compute each of the three terms separately:

\paragraph{e)} Using \cref{as:unbiased_gradients}:

$$\E{\ps{\varepsilon_1(\theta_0)}{ f''(\theta^*) (\eta_0 - \gamma f'(\theta_0) + \gamma \varepsilon_1(\theta_0)} \ | \ \theta_0} = - \gamma \tr{f''(\theta^*) \mathcal{C}(\theta_0)}.$$

\paragraph{f)} Using \cref{as:unbiased_gradients}:

$$\begin{aligned}
\E{\ps{\varepsilon_1(\theta_0)}{f^{(3)}(\theta^*) (\eta_0 - \gamma f'(\theta_0) + \gamma \varepsilon_1(\theta_0)^{\otimes 2}} \ | \ \theta_0} &= \gamma^2 \E{\ps{\varepsilon_1(\theta_0)}{f^{(3)}(\theta^*) \varepsilon_1(\theta_0)^{\otimes 2}} \ | \ \theta_0} \\
& + 2 \gamma \tr{f^{(3)}(\theta^*) (\eta_0 - \gamma f'(\theta_0)) \otimes \varepsilon_1(\theta_0)^{\otimes 2}}.
\end{aligned}$$

\paragraph{g)}  Using the Cauchy-Schwartz inequality:

$$\E{\ps{\varepsilon_1(\theta_0)}{ \mathcal{R}_2(\theta_0  - \gamma f'(\theta_0) + \gamma \varepsilon_1(\theta_0))} \ | \ \theta_0} \leq M_2 \E{\norm{\varepsilon_1(\theta_0)}  \norm{\eta_0  - \gamma f'(\theta_0) + \gamma \varepsilon_1(\theta_0))}^3 }.$$

Such that, taking the expectation under $\theta_0 \sim \pi_{\gamma}$:

\paragraph{e)} 
$\Epi{\ps{\varepsilon_1(\theta_0)}{ f''(\theta^*) (\eta_0 - \gamma f'(\theta_0) + \gamma \varepsilon_1(\theta_0)} } = - \gamma \tr{f''(\theta^*) \Epi{\mathcal{C}(\theta_0)} } . $

\paragraph{f)} Using the Cauchy-Schwartz inequality, \cref{lemma:norm_eta_dieuleveut,lemma:expectation_f_prime}:
$\Epi{\ps{\varepsilon_1(\theta_0)}{f^{(3)}(\theta^*) (\eta_0 - \gamma f'(\theta_0) + \gamma \varepsilon_1(\theta_0))^{\otimes 2}}} = O(\gamma^{3 / 2}).  $

\paragraph{g)} Using the Cauchy-Schwartz inequality and 
$\Epi{\ps{\varepsilon_1(\theta_0)}{ \mathcal{R}_2(\theta  - \gamma f'(\theta_0) + \gamma \varepsilon_1(\theta_0))}} = O(\gamma^{3/2}).$

\paragraph{Putting the terms together.}

Hence gathering \textbf{a)} to  \textbf{g)} together:
$$\Epi{\ps{ f'_1(\theta_0)}{ f'_2(\theta_1)}} = \tr{f''(\theta^*)^2  \Epi{\eta_0 \eta_0^T}} - \gamma \tr{f''(\theta^*) \Epi{\mathcal{C}(\theta_0)} } + O(\gamma^{3 / 2}).$$

We clearly see that $\Epi{\ps{ f'_1(\theta_0)}{ f'_2(\theta_1)}}$ is the sum of a positive value coming from the deterministic term and a negative value due to the noise. We now show that the noise value is typically twice larger than the deterministic value, hence leading to an overall negative inner product. Indeed from Theorem 4 of \citet{dieuleveut2017bridging} we have that $\Epi{\C(\theta_0)} = \C(\theta^*) + O(\gamma)$ and $\Epi{\eta_0 \eta_0^T} = \gamma (f''(\theta^*) \otimes I + I \otimes f''(\theta^*))^{-1} \C(\theta^*) + O(\gamma^2)$.  Hence,
$$\begin{aligned}
\Epi{\ps{ f'_1(\theta_0)}{ f'_2(\theta_1)}} &= \gamma \tr{ f''(\theta^*)^2 (f''(\theta^*) \otimes I + I \otimes f''(\theta^*))^{-1} \C(\theta^*) } \\  
& \qquad - \gamma \tr{f''(\theta^*) C(\theta^*)} + O(\gamma^{3/2}) .
\end{aligned}$$
Notice that $\tr{ f''(\theta^*)^2 (f''(\theta^*) \otimes I + I \otimes f''(\theta^*))^{-1} \C(\theta^*) } = \frac{1}{2} \tr{ f''(\theta^*) \C(\theta^*) }$. We then finally get: 
$$  \Epi{\ps{ f'_1(\theta_0)}{ f'_2(\theta_1)}} = - \frac{1}{2} \gamma \tr{f''(\theta^*) \C(\theta^*)} + O(\gamma^{3/2}).$$ 

\label{appsec:proof_pflug_expectation}

\cref{eq:inner_product_expansion} establishes that  the  sign  of  the expectation of the inner product between two consecutive gradients characterizes the transient and stationary regimes. However, this result does not guarantee the good performance of Pflug's statistic. In fact, as we show in the following section, the statistical test is unable to offer an adequate convergence diagnostic even for simple quadratic functions.

\subsection{Proof of \cref{eq:pflug_probality}} \label{appsec:proof_pflug_probality}

In this subsection we prove \cref{eq:pflug_probality} which shows that in the simple case where $f$ is quadratic and the noise is i.i.d. Pflug's diagnostic does not lead to accurate restarts. We start by stating a few lemmas.

\begin{lemma}
\label{lemma:f(H)}
For $n \geq 0$ we denote $\eta_n = \theta_n - \theta^*$. Let $\eta_0 \in \R^d$, $\Gamma_0 = \eta_0 \eta_0^T$, $\gamma \leq 1 / 2L$ and let $P$ be a polynomial. Under \cref{as:quadratic_f} we have that:

$$\begin{aligned}
\E{\ps{\eta_{n}}{P(H) \eta_n}} &=  \eta_0^T P(H) (I - \gamma H)^{2 n} \eta_0  + \gamma \tr{P(H) C [I - (I - \gamma H)^{2 n}] H^{-1} (2 I - \gamma H)^{-1}} \\
\end{aligned}$$

Therefore when the stationary distribution is reached:

$$\begin{aligned}
\Epi{\ps{\eta}{P(H) \eta}} &= \gamma \tr{C P(H) H^{-1} (2 I - \gamma H)^{-1}} \\
&= \frac{1}{2} \gamma \tr{C P(H) H^{-1}} + o(\gamma).
\end{aligned}$$

\end{lemma}

\begin{proof}

Under \cref{as:quadratic_f} we have that $f'_n(\theta_{n-1}) = H \eta_{n-1} - \xi_{n}$ where the $\xi_n$ are i.i.d. . The SGD recursion becomes:

\begin{align}
\label{eq:semi_sto_sgd_1}
\eta_{n} &= (I - \gamma H) \eta_{n-1} + \gamma \xi_{n} \\
&= (I - \gamma H)^n \eta_0 + \gamma \sum_{k = 1}^n (I - \gamma H)^{n-k} \xi_k. \label{eq:semi_sto_sgd_2}
\end{align}

Since the $(\xi_n)_{n \geq 0}$ are i.i.d. and independent of $\eta_0$ we have that:

$$\begin{aligned}
\E{\ps{\eta_{n}}{P(H) \eta_n}} &=  \eta_0^T P(H) (I - \gamma H)^{2 n} \eta_0 + \gamma^2 \sum_{k = 0}^{n - 1} \E{\xi_{n -k}^T (I - \gamma H)^{2 k} H \xi_{n -k}} \\
&= \eta_0^T P(H) (I - \gamma H)^{2 n} \eta_0 + \gamma^2 \sum_{k = 0}^{n-1} \tr{(I - \gamma H)^{2 k} P(H) \E{ \xi_{n -k} \xi_{n -k}^T }} \\
&= \eta_0^T P(H) (I - \gamma H)^{2 n} \eta_0 + \gamma^2 \tr{ \sum_{k = 0}^{n-1} (I - \gamma H)^{2 k} P(H) C} \\
&= \eta_0^T P(H) (I - \gamma H)^{2 n} \eta_0 + \gamma \tr{C [I - (I - \gamma H)^{2 n}] P(H) H^{-1} (2 I - \gamma H)^{-1}}.
\end{aligned}$$

$\Epi{\ps{\eta}{P(H) \eta}}$ is obtained by taking $n \to + \infty$ in the previous equation.

\end{proof}

The previous lemma holds for $\eta_0 \in \R^d$. We know state the following lemma which assumes that $\theta_0 \sim \pi_{\gamma_{\text{old}}}$.

\begin{lemma}{\label{lemma:restart}}
Let $\gamma_{\text{old}} \leq 1 / 2L$. Assume that $\theta_0 \sim \pi_{\gamma_{\text{old}}}$ and that we start our SGD from that point with a smaller step size $\gamma = r \times \gamma_{\text{old}}$, where $r$ is some parameter in $[0, 1]$. Let $Q$ be a polynomial. Then:

$$\begin{aligned}
\Erestart{\ps{\eta_n}{Q(H) \eta_n} } &= \frac{1}{2} r \gamma_{\text{old}} \left ( \frac{1}{r} - 1 \right ) \tr{Q(H) H^{-1} (I - r \gamma H)^{2 n} C} + \frac{1}{2} r \gamma_{\text{old}} \tr{Q(H) H^{-1} C} + o_n(\gamma) \\
&\leq M \gamma,
\end{aligned}{}$$
where $(\gamma \mapsto \sup_{n \in \mathbb{N}} |o_n(\gamma)|) = o(\gamma)$ and where $M$ is independent of $n$.
\end{lemma}

\begin{proof}
For a step size $\gamma$ we have according to \cref{lemma:f(H)} that:

$$\begin{aligned}
\E{\ps{\eta_n}{Q(H) \eta_n} \ | \ \eta_{0}} &= \eta_{0}^T Q(H) (I - \gamma H)^{2 n} \eta_{0} + \gamma \tr{Q(H) H^{-1} C [ I - (I - \gamma H)^{2 n} ] (2 I - \gamma H)^{-1}} \\
&= \eta_{0}^T Q(H) (I - \gamma H)^{2 n} \eta_{0} +  \frac{1}{2} \gamma \tr{Q(H) H^{-1} C [ I - (I - \gamma H)^{2 n} ] }+ \tilde{o}_n(\gamma).
\end{aligned}$$

Where $\tilde{o}_n(\gamma) = \gamma \tr{Q(H) H^{-1} C [ I - (I - \gamma H)^{2 n} ] [(2 I - \gamma H)^{-1} - \frac{1}{2} I]} = o(\gamma)$ independently of $n \geq 0$. 
Using the second part of \Cref{lemma:f(H)} with $P(H) = Q(H) (I - \gamma H)^{2n}$ we get:
$$\begin{aligned}
\Erestart{\ps{\eta_n}{Q(H) \eta_n}} &= \Erestart{\eta_{0}^T Q(H) (I - \gamma H)^{2 n} \eta_{0}} + \frac{1}{2} \gamma \tr{Q(H) H^{-1} C [ I - (I - \gamma H)^{2 n} ]} + \tilde{o}_n(\gamma) \\
&= \frac{1}{2} \gamma_{\text{old}} \tr{Q(H) H^{-1} C (I - \gamma H)^{2 n}} + \frac{1}{2} \gamma \tr{Q(H) H^{-1} C [ I - (I - \gamma H)^{2 n} ]} + o(\gamma) + \tilde{o}_n(\gamma) \\
&= \frac{1}{2} \gamma (\frac{1}{r} - 1) \tr{Q(H) H^{-1} (I - r \gamma_{\text{old}} H)^{2 n} C} + \frac{1}{2} \gamma \tr{Q(H) H^{-1} C} + o_n(\gamma),
\end{aligned}$$
where $o_n(\gamma) = \tilde{o}_n(\gamma) + o(\gamma)$. This immediately gives that $\Erestart{\ps{\eta_n}{Q(H) \eta_n}} \leq M \gamma$.
\end{proof}

\paragraph{Back to Plug's statistic.}

Under \cref{as:quadratic_f}, we have that $f'_{k + 1}(\theta_{k}) = H \eta_{k} - \xi_{k+1}$, $f'_{k + 2}(\theta_{k+1}) = H \eta_{k+1} - \xi_{k+2}$ and  $\eta_{k+1} = (I - \gamma H) \eta_{k} + \gamma \xi_{k+1}$. Thus,
\begin{align}
\ps{f'_{k+1}(\theta_k)}{f'_{k+2}(\theta_{k+1})} &= \ps{ H \eta_{k} - \xi_{k+1}}{ H \eta_{k+1} - \xi_{k+2}} \nonumber \\
&= \ps{ H \eta_{k} - \xi_{k+1}}{ H ( I - \gamma H) \eta_{k} + \gamma H \xi_{k+1} - \xi_{k+2}} \nonumber \\
&= \tr{ \Big [ H^2 (I- \gamma H)  \eta_k \eta_k^T - H  \xi_{k+2} \eta_k^T  
- H (I - 2 \gamma H) \xi_{k+1} \eta_k^T   \nonumber \\
& \hspace{3cm}
- \gamma H  \xi_{k+1} \xi_{k+1}^T \Big ]} 
+ \xi_{k+1} \xi_{k+2} . \nonumber
\end{align}
Hence,
\begin{align}
S_n &= \frac{1}{n} \sum_{k = 0}^{n-1} \ps{f'_{k+1}}{f'_{k+2}} \nonumber  \\
&= \frac{1}{n} \Big [ \tr{H^2 (I- \gamma H)  \sum_{k = 0}^{n-1} \eta_k \eta_k^T}  - \tr{H \sum_{k = 0}^{n-1} \xi_{k+2} \eta_k^T} 
- \tr{H (I - 2 \gamma H) \sum_{k = 0}^{n-1} \xi_{k+1} \eta_k^T} \nonumber \\ 
& \hspace{3cm}
- \gamma \tr{H \sum_{k = 0}^{n-1} \xi_{k+1} \xi_{k+1}^T} 
+ \sum_{k = 0}^{n-1} \xi_{k+1} \xi_{k+2} \Big ].  \label{app_eq:S_n}
\end{align}

Let us define 
\begin{equation}\label{eq:chi}
\chi_n = \frac{1}{n} \sum_{k = 0}^{n-1} \xi_{k+1}^T \xi_{k+2},\end{equation}
notice that $\chi_n$ is independent of $\gamma$. Let also denote by 
\begin{align}
R_n^{(\gamma)} &= - \frac{1}{n} \Big [ \tr{H^2 (I- \gamma H)  \sum_{k = 0}^{n-1} \eta_k \eta_k^T}  - \tr{H \sum_{k = 0}^{n-1} \xi_{k+2} \eta_k^T} \nonumber
 \\
& \hspace{3cm} - \tr{H (I - 2 \gamma H) \sum_{k = 0}^{n-1} \xi_{k+1} \eta_k^T}
- \gamma \tr{H \sum_{k = 0}^{n-1} \xi_{k+1} \xi_{k+1}^T} \Big ] \nonumber \\
&= - \frac{1}{n} \left [ T_{1, n}^{(\gamma)} + T_{2, n}^{(\gamma)} + T_{3, n}^{(\gamma)} + T_{4, n}^{(\gamma)} \right ]. \label{eq:snn}
\end{align}
where $T_{1, n}^{(\gamma)}$, $T_{2, n}^{(\gamma)}$, $T_{3, n}^{(\gamma)}$ and $T_{4, n}^{(\gamma)}$ are defined in the respective order from the previous line.
Then \cref{app_eq:S_n} can be written as:
$$S_n = - R_n^{(\gamma)} + \frac{1}{n} \sum_{k = 0}^{n-1} \xi_{k+1} \xi_{k+2} = - R_n^{(\gamma)} + \chi_n.$$

We now state the following lemma which is crucial in showing \cref{eq:pflug_probality}. Indeed \cref{lemma:var_R} shows that though the signal $R_n^{(\gamma)}$ is positive after a restart, it is typically of order $O(\gamma)$.

\begin{lemma}{\label{lemma:var_R}} Let us consider $R_n^{(\gamma)}$ defined in \cref{eq:snn}.
Assume that $\theta_0 = \theta_{restart} \sim \pi_{\gamma_{\text{old}}}$ and that we start our SGD from that point with a smaller step size $\gamma = r \times \gamma_{\text{old}}$, where $r$ is some parameter in $[0, 1]$. Then,
$$\Erestart{{R_n^{(\gamma)}}^2} \leq M \left (\frac{\gamma}{n} + \gamma^2 \right),$$
where $M$ does not depend neither of $\gamma$ nor of $n$.
\end{lemma}

\begin{proof}
In the proof we consider separately $T_{1, n}^{(\gamma)}, \ \ldots \ , T_{4, n}^{(\gamma)}$ and then use the fact that $(a + b + c + d)^2 \leq 4 (a^2 + b^2 + c^2 + d^2)$.

\begin{itemize}

\item \textbf{$T_{1, n}^{(\gamma)}$:} Let $P(H) = H^2 (I - \gamma H)$:
$$\Erestart{{T_{1, n}^{(\gamma)}}^2} = \sum_{k, k' = 0}^{n-1} \Erestart{\eta_{k}^T P(H) \eta_{k} \eta_{k'}^T P(H) \eta_{k'}}.$$

Let $\Tilde{\eta}_k = P(H)^{1/2} \eta_k$, then: 
$$\Erestart{{T_{1, n}^{(\gamma)}}^2} = \sum_{k, k' = 0}^{n-1} \Erestart{\norm{\Teta_k}^2 \norm{\Teta_{k'}}^2}.$$

Let $D_k = (I - \gamma H)^k \eta_0$ be the deterministic part and $S_k = \gamma \sum_{i = 0}^{k-1} (I - \gamma H)^i \xi_{k-i}$ the stochastic one. From \cref{eq:semi_sto_sgd_2}: $\Teta_k = P(H)^{1/2} (D_k + S_k)$, hence $\norm{\Teta_k}^2 \leq 2 \norm{P(H)^{1/2}}_{\text{op}}^2 (\norm{D_k}^2 + \norm{S_k}^2)$. Let $ C_1^{(0)} = 2 \norm{P(H)^{1/2}}_{\text{op}}^2$, then:

$$\begin{aligned}
\Erestart{{T_{1, n}^{(\gamma)}}^2} &\leq C_1^{(0)} \sum_{k, k' = 0}^{n-1} \Erestart{(\norm{D_k}^2 + \norm{S_k}^2) (\norm{D_{k'}}^2 + \norm{S_{k'}}^2)} \\
&\leq C_1^{(0)} \bigg (\sum_{k, k' = 0}^{n-1} \Erestart{\norm{D_k}^2 \norm{D_{k'}}^2} +  2 \sum_{k, k' = 0}^{n-1} \Erestart{\norm{D_k}^2 \norm{S_{k'}}^2} \\
& \hspace{3cm} + \sum_{k, k' = 0}^{n-1} \E{\norm{S_k}^2 \norm{S_{k'}}^2} \bigg ).
\end{aligned}$$

However:
$$\begin{aligned}
\sum_{k, k' = 0}^{n-1}  \Erestart{\norm{D_k}^2 \norm{D_{k'}}^2} &\leq \sum_{k, k' = 0}^{n-1}  \Erestart{\norm{I - \gamma H}_{\text{op}}^{2 (k + k')}  \norm{\eta_0}^4} \\
&\leq n^2 \Erestart{\norm{\eta_0}^4} \qquad \text{since } \norm{I - \gamma H}_{\text{op}} \leq 1 \\
&\leq \Tilde{C}_1^{(1)}  n^2 \gamma^2 \qquad \text{(according to  \Cref{lemma:norm_eta_dieuleveut})}. 
\end{aligned}$$

Notice that $\Erestart{\norm{D_k}^2} \leq \Erestart{\norm{\eta_0}^2} = O(\gamma)$ (independently of $k$) according to \Cref{lemma:norm_eta_dieuleveut} and $\E{\norm{S_k}^2} = O(\gamma)$ (independently of $k$) according to \cref{lemma:f(H)} with $\eta_0 = 0$ and $P = 1$. Hence using the fact that the $(\xi_n)_{n \geq 0}$ are independent of $\eta_0$:

$$\begin{aligned}
\sum_{k, k' = 0}^{n-1} \Erestart{\norm{D_k}^2 \norm{S_{k'}}^2} &= \sum_{k = 0}^{n-1} \Erestart{\norm{D_k}^2} \sum_{k'=0}^{n-1} \E{\norm{S_{k'}}^2} \\
&\leq O( (n \gamma) \times (n \gamma)) \\
&\leq \Tilde{C}_1^{(2)} n^2 \gamma^2.
\end{aligned}$$

Assume w.l.o.g. that $k \leq k'$, let $\Delta_k = (k' - k)$:
$$\begin{aligned}
\E{\norm{S_k}^2 \norm{S_{k'}}^2} &= \gamma^4 \E{ \sum_{1 \leq i, j \leq k \atop 1 \leq l, p \leq k'} \xi_{i}^T (I - \gamma H)^{2 k - (i + j)} \xi_{j} \xi_{l}^T (I - \gamma H)^{2 k' - (l + p)} \xi_{p}}.
\end{aligned}$$

To compute the sum over the four indices we distinguish the three cases where the expectation is not equal to $0$:

First case, $i = j = l = p$:
$$\begin{aligned}
&\E{ \sum_{1 \leq i \leq k} \xi_{i}^T (I - \gamma H)^{2 k - 2 i} \xi_{i} \xi_{i}^T (I - \gamma H)^{2 k - 2 i} \xi_{i}} \\
& \hspace{3cm} =  \sum_{1 \leq i \leq k} \tr{\E{ (I - \gamma H)^{2 (k - i)} \xi_{i} \xi_{i}^T (I - \gamma H)^{2 (k' - i)} \xi_{i} \xi_{i}^T }} \\ 
& \hspace{3cm} =  d \times \sum_{1 \leq i \leq k} \norm{ \E{ (I - \gamma H)^{2 (k - i)} \xi_{i} \xi_{i}^T (I - \gamma H)^{2 (k' - i)} \xi_{i} \xi_{i}^T }}_{\text{op}} \\ 
& \hspace{3cm} \leq  d \times \sum_{1 \leq i \leq k} \E{ \norm{ (I - \gamma H)}_{\text{op}}^{2 (k - i)} \norm{ (I - \gamma H)}_{\text{op}}^{2 (k' - i)} \norm{ \xi_{i} \xi_{i}^T}_{\text{op}}^2} \\ 
& \hspace{3cm} \leq  d \times \E{\norm{ \xi_{1}}^4} \norm{ I - \gamma H}_{\text{op}}^{2 \Delta_k}  \sum_{1 \leq i \leq k} \norm{ I - \gamma H}_{\text{op}}^{2 i} \\ 
& \hspace{3cm} \leq  \Tilde{C}_1^{(3)}  \frac{1}{1 - \norm{ I - \gamma H}_{\text{op}}^{2} } \quad \text{where } \Tilde{C}_1^{(3)} = d \times \E{\norm{ \xi_{1}}^4}  \\
& \hspace{3cm} \leq \Tilde{C}_1^{(3)}  \frac{1}{1 - \norm{ I - \gamma H}_{\text{op}} } \\
& \hspace{3cm} = \Tilde{C}_1^{(3)}  \frac{1}{\gamma \mu}  \\
& \hspace{3cm} \leq \Tilde{C}_1^{(3)}  \frac{1}{\gamma^2 \mu^2 } \\
& \hspace{3cm}  = \Tilde{C}_1^{(4)} \frac{1}{\gamma^2} \quad \text{where } \Tilde{C}_1^{(4)} = \Tilde{C}_1^{(3)} \mu^{-2} .
\end{aligned}$$

Second case, $i = j$,  $l = p$:
$$\begin{aligned}
&\E{ \sum_{1 \leq i \leq k \atop 1 \leq l \leq k', i \neq l} \xi_{i}^T (I - \gamma H)^{2 (k - i)} \xi_{i} \xi_{l}^T (I - \gamma H)^{2 (k' - l)} \xi_{l}} \\
& \hspace{3cm} \leq \sum_{1 \leq i \leq k } \E{ \xi_{i}^T (I - \gamma H)^{2 (k - i)} \xi_{i}}  \sum_{1 \leq l \leq k' } \E{ \xi_{l}^T (I - \gamma H)^{2 (k' - l)} \xi_{l}}  \\
& \hspace{3cm} \leq \sum_{1 \leq i \leq k } \tr{(I - \gamma H)^{2 (k - i)} C}  \sum_{1 \leq l \leq k' } \tr{ (I - \gamma H)^{2 (k' - l)} C}  \\
& \hspace{3cm} \leq d^2 \norm{C}_{\text{op}}^2   \sum_{1 \leq i \leq k } \norm{(I - \gamma H)}_{\text{op}}^{2 (k - i)}   \sum_{1 \leq l \leq k' } \norm{(I - \gamma H)}_{\text{op}}^{2 (k' - l)}   \\
& \hspace{3cm} \leq \Tilde{C}_1^{(5)} \frac{1}{\gamma^2} \quad \text{where } \Tilde{C}_1^{(5)} = d^2 \norm{C}_{\text{op}}^2  \mu^{-2} .
\end{aligned}$$

Third case, $i = p$,  $j = l$:
$$\begin{aligned}
&\E{ \sum_{1 \leq i \leq k \atop 1 \leq j \leq k, i \neq j} \xi_{i}^T (I - \gamma H)^{2 k - (i+j)} \xi_{j} \xi_{j}^T (I - \gamma H)^{2 k' - (i+j)} \xi_{i}} \\
& \hspace{2cm} = \tr{ \E{ \sum_{1 \leq i \leq k \atop 1 \leq j \leq k, i \neq j}  (I - \gamma H)^{2 k - (i+j)} \xi_{j} \xi_{j}^T (I - \gamma H)^{2 k' - (i+j)} \xi_{i} \xi_{i}^T}} \\
& \hspace{2cm} = \tr{ \sum_{1 \leq i \leq k \atop 1 \leq j \leq k, i \neq j}  \E{  (I - \gamma H)^{2 k - (i+j)} \xi_{j} \xi_{j}^T} \E{ (I - \gamma H)^{2 k' - (i+j)} \xi_{i} \xi_{i}^T}} \\
& \hspace{2cm} = \tr{ \sum_{1 \leq i \leq k \atop 1 \leq j \leq k, i \neq j}   (I - \gamma H)^{2 k - (i+j)} C (I - \gamma H)^{2 k' - (i+j)} C } \\
& \hspace{2cm} \leq d \times \sum_{1 \leq i \leq k \atop 1 \leq j \leq k, i \neq j} \norm{I - \gamma H}_{\text{op}}^{2 [(k + k') - (i+j)]} \norm{C}_{\text{op}}^2  \\
& \hspace{2cm} \leq d \norm{C}_{\text{op}}^2 \sum_{1 \leq i \leq k} \norm{I - \gamma H}_{\text{op}}^{2 (k - i)}  \sum_{1 \leq j \leq k} \norm{I - \gamma H}^{2 (k' - j)} \\
& \hspace{2cm} \leq d \norm{C}_{\text{op}}^2 \sum_{1 \leq i \leq k} \norm{I - \gamma H}_{\text{op}}^{2 (k - i)}  \sum_{1 \leq j \leq k} \norm{I - \gamma H}^{2 (k - j)} \\
& \hspace{2cm} \leq \Tilde{C}_1^{(6)} \frac{1}{\gamma^2} \quad \text{where } \Tilde{C}_1^{(6)} = d \norm{C}_{\text{op}}^2  \mu^{-2} .
\end{aligned}$$

Therefore with $\Tilde{C}_1^{(7)} = \Tilde{C}_1^{(4)} + \Tilde{C}_1^{(5)} + \Tilde{C}_1^{(6)}$ we get that $\E{\norm{S_k}^2 \norm{S_{k'}}^2} \leq \Tilde{C}_1^{(7)} \gamma^4 \times \frac{1}{\gamma^2}$ independently of $k$  and 
$$\sum_{k, k' = 0}^{n-1} \E{\norm{S_k}^2 \norm{S_{k'}}^2} \leq  \Tilde{C}_1^{(7)} n^2 \gamma^2.$$
Finally let $C_1 = \Tilde{C}_1^{(0)} \times ( \Tilde{C}_1^{(1)} + \Tilde{C}_1^{(2)} + \Tilde{C}_1^{(7)})$, then, 
$$\Erestart{{T_{1, n}^{(\gamma)}}^2} \leq C_1 n^2 \gamma^2.$$ 

\item \textbf{$T_{2, n}^{(\gamma)}$:} By independence of the $(\xi_k)_{k \geq 0}$ and by \Cref{lemma:restart} with $Q(H) = H C H$:
$$\Erestart{{T_{2, n}^{(\gamma)}}^2} = \sum_{k = 0}^{n-1} \Erestart{(\xi_{k+2}^T H \eta_{k})^2} = \sum_{k = 0}^{n-1} \Erestart{\eta_{k}^T H C H \eta_{k}} \leq \sum_{k = 0}^{n-1} C_2 \gamma = C_2 n \gamma .$$

\item \textbf{$T_{3, n}^{(\gamma)}$:} With the same reasoning as $T_{2, n}^{(\gamma)}$ we get:

$$\Erestart{{T_{3, n}^{(\gamma)}}^2} \leq C_3 n \gamma.$$

\item \textbf{$T_{4, n}^{(\gamma)}$:} By independence of the $(\xi_k)_{k \geq 0}$:
$$\Erestart{{T_{4, n}^{(\gamma)}}^2} = \gamma^2  \sum_{k = 0}^{n-1} \Erestart{(\xi_{k+1}^T H \xi_{k+1})^2} \leq C_4 n \gamma^2.$$ 
\end{itemize}

Putting everything together we obtain:
$$\Erestart{R_n^{(\gamma)}} \leq M \left ( \frac{\gamma}{n} + \gamma^2 \right ).$$

\end{proof}

Contrary to $R_n^{(\gamma)}$ we now show that though the noise $\chi_n$ is in expectation equal to $0$, it has moments which are independent of $\gamma$.

\begin{lemma}{\label{Var_xi}}
Let us consider $\chi_n$ defined in \cref{eq:chi}. Then we have 
$$\Var{\chi_n} = \frac{1}{n} \tr{(C^2)} \qquad \text{and} \qquad \Var{\chi_n^2} = \frac{\E{(\xi_1^T \xi_2)^4} - \trsq{C^2}}{n^3}.$$
\end{lemma}

\begin{proof}
$$\begin{aligned}
\Var{\chi_n} &= \frac{1}{n^2} \sum_{i, j = 0}^{n-1} \mathrm{Cov}(\xi_{i+1}^T \xi_{i+2}, \xi_{j+1}^T \xi_{j +2}) \\
&= \frac{1}{n} \sum_{i = 0}^{n-1} \mathrm{Var}(\xi_{i+1}^T \xi_{i+2})  \\
&= \frac{1}{n} \tr{(C^2)}.
\end{aligned}$$

$$\begin{aligned}
\E{\chi_n^4} &= \frac{1}{n^4} \sum_{i, j, k, l= 0}^{n-1} \E{\xi_i^T \xi_{i+1} \xi_j^T \xi_{j+1} \xi_k^T \xi_{k+1} \xi_l^T \xi_{l+1}} \\
&= \frac{1}{n^4} \left ( \sum_{i= 0}^{n-1} \E{(\xi_i^T \xi_{i+1})^4} + \sum_{i, j = 0 \atop i \neq j}^{n-1} \E{(\xi_i^T \xi_{i+1})^2 (\xi_j^T \xi_{j+1})^2} \right )\\
&= \frac{1}{n^3} \E{(\xi_1^T \xi_{2})^4} + \frac{n (n-1)}{n^4} \E{(\xi_1^T \xi_{2})^2}^2\\
&= \frac{\E{(\xi_1^T \xi_{2})^4} - \trsq{C^2}}{n^3}  + \frac{\trsq{C^2}}{n^2}.
\end{aligned}$$
Therefore:
$$\Var{\chi_n^2} = \frac{\E{(\xi_1^T \xi_2)^4} - \trsq{C^2}}{n^3}.$$
\end{proof}

We know show that under the symmetry \cref{as:noise_symmetry}, we can easily control $\Prob{S_{n} \leq 0} = \Prob{\chi_{n} \leq R_{n}^{(\gamma)}}$ by probabilities involving the square of the variables. These probabilities are then be easy to control using the Markov inequality and Paley-Zigmund's inequality.

\begin{lemma}\label{lemma:probability_inequality}
Let $c_{\gamma} > 0$, let $\chi_n$ be a real random variable that verifies $\forall x \geq 0$, $\Prob{\chi_n \geq x} = \Prob{\chi_n \leq -x}$, let $R_n^{(\gamma)}$ be a real random variable.  Then:
$$ \begin{aligned}
\frac{1}{2} \Prob{\chi_n^2 \geq  c_{\gamma}^2 } - \Prob{{R_n^{(\gamma)}}^2 \geq c_{\gamma}^2}
&\leq \Prob{\chi_n \leq R_n^{(\gamma)}} \\
&\leq  1 - \frac{1}{2} \Prob{\chi_n^2 \geq c_{\gamma}^2}  + \Prob{{R_n^{(\gamma)}}^2 \geq c_{\gamma}^2} .
\end{aligned}$$
\end{lemma}

\begin{proof}

Notice the inclusion $\left \{ \chi_n \leq - c_{\gamma} \right \} \cap \left \{ | R_n^{(\gamma)} | \leq c_{\gamma} \right \} \subset \left \{ \chi_n \leq R_n^{(\gamma)}  \right \} $. Furthermore, for two random events $A$ and $B$ we have that $\Prob{A \cap B} = \Prob{A \setminus B^c} \geq \Prob{A} - \Prob{B^c} $. Hence:

\begin{align*}
\Prob{\chi_n \leq R_n^{(\gamma)}}  &\geq \Prob{ \chi_n \leq - c_{\gamma},  | R_n^{(\gamma)} | \leq c_{\gamma} } \\
&\geq \Prob{\chi_n \leq - c_{\gamma}} - \Prob{| R_n^{(\gamma)} | > c_{\gamma} } 
\end{align*}

However the symmetry assumption on $\chi_n$ implies that $\Prob{\chi_n \leq - c_{\gamma}} =  \Prob{\chi_n \geq  c_{\gamma}} = \frac{1}{2} \Prob{\chi_n^2 \geq  c_{\gamma}^2}$. Notice also that $\Prob{| R_n^{(\gamma)} | > c_{\gamma} } = \Prob{ {R_n^{(\gamma)}}^2  > c_{\gamma}^2 }$. Hence:
\begin{align*}
\Prob{\chi_n \leq R_n^{(\gamma)}}  &\geq \frac{1}{2} \Prob{\chi_n^2 \geq  c_{\gamma}^2 } - \Prob{{R_n^{(\gamma)}}^2 \geq c_{\gamma}^2}
\end{align*}

For the upper bound, notice that $\left \{ \chi_n \leq R_n^{(\gamma)}  \right \} \subset  \left \{ \chi_n <   c_{\gamma} \right \} \cup  \left \{ | R_n^{(\gamma)} | \geq c_{\gamma} \right \}$
Hence:

\begin{align*}
\Prob{ \chi_n \leq R_n^{(\gamma)}  } &\leq \Prob{\chi_n <   c_{\gamma} } + \Prob{| R_n^{(\gamma)} | \geq c_{\gamma}} \\ 
&\leq 1 - \Prob{\chi_n \geq   c_{\gamma} } + \Prob{| R_n^{(\gamma)} | \geq c_{\gamma}} \\ 
&= 1 - \frac{1}{2} \Prob{\chi_n^2 \geq c_{\gamma}^2}  + \Prob{{R_n^{(\gamma)}}^2 \geq c_{\gamma}^2} .
\end{align*}
\end{proof}
We now prove \cref{eq:pflug_probality}. To do so we distinguish two cases, the first one corresponds to $\alpha = 0$, the second to $0 < \alpha < 2$.
\paragraph{Proof of \cref{eq:pflug_probality}.}

\paragraph{First case: $\alpha = 0$, $n_{\gamma} = n_b$.}
For readability reasons we will note $\mathbb{P} = \mathbb{P}_{\theta_0 \sim \pi_{\gamma_{\text{old}}}}$. 
Notice that:
$$\begin{aligned}
\Prob{S_{n_b} \leq 0} &= \Prob{\chi_{n_b} \leq R_{n_b}^{(\gamma)}}.
\end{aligned}$$

Let $c_{\gamma} = \gamma^{1/4}$. By the continuity assumption from \cref{as:noise_condition}: $\Prob{\chi_{n_b}^2 \geq c_{\gamma}^2} \underset{\gamma \to 0}{\longrightarrow} \Prob{\chi_{n_b}^2 \geq 0} = 1$. On the other hand, according to \Cref{lemma:var_R}, $\Erestart{{R_{n_b}^{(\gamma)}}^2} = O(\gamma)$. Therefore by Markov's inequality:
$$\ProbRest{{{R_{n_b}^{(\gamma)}}}^2 \geq c_{\gamma}^2} \leq \frac{\Erestart{{R_{n_b}^{(\gamma)}}^2}}{c_{\gamma}^2} = \gamma^{- 1/2} \times O(\gamma) \underset{\gamma \to 0}{\longrightarrow} 0.$$

Finally we get that:
$$\frac{1}{2} \Prob{\chi_n^2 \geq  c_{\gamma}^2 } - \Prob{{R_n^{(\gamma)}}^2 \geq c_{\gamma}^2} \underset{\gamma \to 0}{\longrightarrow} \frac{1}{2}.$$
and
$$ 1 - \frac{1}{2} \Prob{\chi_n^2 \geq c_{\gamma}^2} \  + \Prob{{R_n^{(\gamma)}}^2 \geq c_{\gamma}^2}  \underset{\gamma \to 0}{\longrightarrow} \frac{1}{2}.$$

By \Cref{lemma:probability_inequality}:
$$\ProbRest{S_{n_b} \leq 0} \underset{\gamma \to 0}{\longrightarrow} \frac{1}{2}.$$

\paragraph{Second case: $0 < \alpha < 2$.}

For $\alpha > 0$ we make use of the following lemma~\citep{paley1932series}.
\begin{lemma}[Paley-Zigmund inequality]

Let $Z \geq 0$ be a random variable with finite variance and $\theta \in [0, 1]$, then:
$$\Prob{Z > \theta \E{Z}} \geq \frac{(1 - \theta)^2 \E{Z}^2}{\Var{Z} + (1 - \theta)^2 \E{Z}^2}.$$
\end{lemma}

We can now prove \cref{eq:pflug_probality} when $\alpha \neq 0$.

For readability reasons we note $\mathbb{P} = \mathbb{P}_{\theta_0 \sim \pi_{\gamma_{\text{old}}}}$. We follow the same reasoning as in the case $\alpha = 0$. However in this case we need to be careful with the fact that $n$ depends on $\gamma$.

Notice that: 
$$\begin{aligned}
\Prob{S_{n_{\gamma}} \leq 0} &= \Prob{\chi_{n_{\gamma}} \leq R_{n_{\gamma}}^{(\gamma)}}.
\end{aligned}$$

Let $c_{\gamma} = \gamma^{(\alpha + 1) / 3}$ and let $\theta_{n_{\gamma}}^{(\gamma)} = (n_{\gamma} \times c_{\gamma}^2) / \tr{C^2}$. By \Cref{Var_xi}, we have that $\E{\chi_n^2} = \frac{1}{n} \tr{C^2}$, therefore:
$$\Prob{\chi_{n_{\gamma}}^2 \geq c_{\gamma}^2} = \Prob{\chi_{n_{\gamma}}^2 \geq  \E{\chi_{n_{\gamma}}^2} \times \theta_{n_{\gamma}}^{(\gamma)}},$$

Notice that $n_{\gamma} \times c_{\gamma}^2 = O (\gamma^{(2 - \alpha)/3)})$. Since $\alpha < 2$, we have that $\theta_{n_{\gamma}}^{(\gamma)} \underset{\gamma \to 0}{\longrightarrow} 0$. Therefore by the Paley-Zigmund inequality (valid since $\theta_{n_{\gamma}}^{(\gamma)} < 1$ for $\gamma$ small enough):

$$\Prob{\chi_{n_{\gamma}}^2 > \E{\chi_{n_{\gamma}}^2} \times \theta_{n_{\gamma}}^{(\gamma)}} \geq \frac{(1 - \theta_{n_{\gamma}}^{(\gamma)}) \E{\chi_{n_{\gamma}}^2}^2}{\Var{\chi_{n_{\gamma}}^2} + (1 - \theta_{n_{\gamma}}^{(\gamma)})  \E{\chi_{n_{\gamma}}^2}^2}.$$

By \Cref{Var_xi}, $\E{\chi_n^2} = \frac{1}{n} \tr{C^2}$ and $\Var{\chi_n^2} = B / n^3$, therefore since $n_{\gamma} \underset{\gamma \to 0}{\longrightarrow} + \infty$ we get that $\Var{\chi_{n_{\gamma}}^2} \underset{\gamma \to 0}{=} o \left (\E{\chi_{n_{\gamma}}^2}^2 \right)$. Moreover $(1 - \theta^{(\gamma)}_{n_{\gamma}}) \underset{\gamma \to 0}{\to} 1$. Therefore: 

$$\frac{(1 - \theta_{n_{\gamma}}^{(\gamma)}) \E{\chi_{n_{\gamma}}^2}^2}{\Var{\chi_{n_{\gamma}}^2} + (1 - \theta_{n_{\gamma}}^{(\gamma)})  \E{\chi_{n_{\gamma}}^2}^2} \underset{\gamma \to 0}{\longrightarrow} 1.$$

Therefore $\Prob{\chi_{n_{\gamma}}^2 > c_{\gamma}^2} \underset{\gamma \to 0}{\longrightarrow} 1$.
On the other hand, according to \Cref{lemma:restart}:
$$\mathbb{E}_{\theta_0 \sim \pi_{\gamma_{\text{old}}}}[{R_{n_{\gamma}}^{(\gamma)}}^2] \leq M (\frac{\gamma}{n_{\gamma}} + \gamma^2) = O(\max(\gamma^{(1 + \alpha)}, \gamma^2))$$

Using Markov's inequality: $$\ProbRest{{{R_{n_{\gamma}}^{(\gamma)}}}^2 \geq c_{\gamma}^2} \leq \frac{\Erestart{{R_{n_{\gamma}}^{(\gamma)}}^2}}{c_{\gamma}^2} = O(\max(\gamma^{(\alpha + 1)/ 3)}, \gamma^{\frac{2}{3} (2 - \alpha))}) \underset{\gamma \to 0}{\longrightarrow} 0$$

Finally, using the inequalities form \Cref{lemma:probability_inequality} we get:
$$\ProbRest{S_{n_{\gamma}} \leq 0} \underset{\gamma \to 0}{\longrightarrow} \frac{1}{2}.$$

\paragraph{Remark:}  For the case $\alpha = 0$, if $x \in \R_+ \longmapsto f(x) = \Prob{\chi_{n_b}^2 \geq x}$ is not continuous in $0$ as needed to show the result we can then follow the exact the same proof as when $\alpha > 0$ but with $\alpha = 0$. However we cannot use the fact that $\Var{\chi_{n_{b}}^2} \underset{\gamma \to 0}{=} o \left ( \E{\chi_{n_{b}}^2}^2 \right)$ but we still get by using Paley-Zigmund's inequality that: 
$$\frac{(1 - \theta_{n_b}^{(\gamma)}) \E{\chi_{n_b}^2}^2}{\Var{\chi_{n_b}^2} + (1 - \theta_{n_b}^{(\gamma)})  \E{\chi_{n_b}^2}^2} \underset{\gamma \to 0}{\longrightarrow} \frac{\E{\chi_{n_b}^2}^2}{\E{\chi_{n_b}^4}},$$
which then leads to: 
$$\ProbRest{S_{n_b}^{(\gamma)} \leq 0} \underset{\gamma \to 0}{\longrightarrow} \frac{1}{2} \frac{\E{\chi_{n_b}^2}^2}{\E{\chi_{n_b}^4}}.$$

 \subsection{Problem with the proof by \citet{pflug1988adaptive}.}
\label{appsec:proof_pflug_mistake}
 There is a mistake inequality (21) in the proof of the main result of \citet{pflug1988adaptive}. Indeed they compute $\Var{S_n}$ but forget the terms $\Var{\xi_{i}^T \xi_{i + 1}}$ which are independent of $\gamma$. Hence it is not true that $\Var{S_n} = O(\gamma)$ as they state.

\section{Proof for the distance-based statistic}\label{appsec:new_stat}

We prove here \cref{eq:omega} and \cref{corr:slopes}. Since we have from \cref{eq:semi_sto_sgd_2}:
\[\eta_{n} = (I - \gamma H)^n \eta_0 + \gamma \sum_{k = 1}^n (I - \gamma H)^{n-k} \xi_k,
\]
it immediately implies that  
 \[\Omega_{n} = \eta_n - \eta_0 = [(I - \gamma H)^n - I] \eta_0 + \gamma \sum_{k = 1}^n (I - \gamma H)^{n-k} \xi_k.\]
 Taking the expectation of the square norm and using the fact that $(\xi_i)_{i \geq 0}$ are i.i.d. and independent of $\theta_0$ we get:
$$\E{\norm{\Omega_{n}}^2} = \eta_0^T [I - (I - \gamma H)^n]^2 \eta_0 + \gamma \tr{[I - (I - \gamma H)^{2 n}] (2 I - \gamma H)^{-1} H^{-1} C}.$$
Hence by taking $n$ to infinity: 
$$\Epi{\norm{\Omega_{n}}^2} = \norm{\eta_0}^2 + \gamma \tr{H^{-1} C (2I - \gamma H)^{-1}}.$$
and by a Taylor expansion for $(n \gamma)$ small:
$$\E{\norm{\Omega_{n}}^2} = \gamma^2 \eta_0^T H^2 \eta_0 \times n^2 + \gamma^2 \tr{C} \times n + o((n \gamma)^2).$$ 
These two last equalities conclude the proof.

\end{document}